%% file: paper.tex
\documentclass{article}
\usepackage{iclr2027_conference,times}
\usepackage{amsthm}
\usepackage{amsmath}
\usepackage{amssymb}
\usepackage{algorithm}
\usepackage{algorithmic}
\usepackage{graphicx}
\usepackage{natbib}
\usepackage{booktabs,makecell,pifont}

\usepackage{hyperref}
\usepackage{url}
\usepackage{comment}
\usepackage{mathtools}

\newcommand{\argmin}{\operatorname*{argmin}}
\newcommand{\argmax}{\operatorname*{argmax}}
\newcommand{\R}{\mathbb R}
\newcommand{\E}{\mathbb E}

\newcommand{\one}{\mathbf 1}
\newcommand{\op}{\mathrm{op}}
\newcommand{\softmax}{\operatorname{softmax}}

\newcommand{\triplenorm}[1]{%
  \left\lvert\!\left\lvert\!\left\lvert
  #1
  \right\rvert\!\right\rvert\!\right\rvert
}

\newcommand{\norm}[1]{\left\lVert #1\right\rVert}

\newcommand{\traceinner}[2]{%
  \left\langle\!\left\langle
  #1,\,#2
  \right\rangle\!\right\rangle
}

\newcommand{\rank}{\mathrm{rank}}
\newcommand{\tr}{\mathrm{tr}}
\newcommand{\diag}{\mathrm{diag}}
\newcommand{\Var}{\mathrm{Var}}

\newcommand{\Pp}{\mathbb P}

\newcommand{\StateSpace}{\mathcal{S}}

\newcommand{\ActionSpace}{\mathcal{A}}

\newcommand{\Transition}{\mathsf P}

\newcommand{\InitialDist}{d_0}

\newtheorem{lemma}{Lemma}
\newtheorem{theorem}{Theorem}

\newtheorem{assumption}{Assumption}
\newtheorem{proposition}{Proposition}

\theoremstyle{remark}
\newtheorem{remark}[lemma]{Remark}
\newtheorem{example}[lemma]{Example}
\newtheorem{claim}[lemma]{Claim}

\newcommand{\rankbound}{k}
\newcommand{\softvalue}{J}

\author{Allen Tran$^{1}$\thanks{Equal contribution.} \quad
Jia Wan$^{2}$\footnotemark[1] \quad
Nathan Kallus$^{1,3}$ \quad
 Aur\'elien Bibaut$^{1}$\\
$^1$Netflix \quad $^2$Massachusetts Institute of Technology \quad $^3$Cornell Tech \\
\texttt{\{allent,abibaut,nkallus\}@netflix.com}, \texttt{jiawan@mit.edu}}

\title{Scalable Multi-Task Inverse Reinforcement Learning}

\iclrfinalcopy 
\begin{document}

\maketitle
\begin{abstract}
By learning transferable rewards, inverse reinforcement learning (IRL) enables counterfactual evaluation of agents under modified environments. Such transfer places strict requirements on coverage since target environments affect agents' state occupancy. We propose a multi-task IRL method that pools data across multiple agents with different rewards in the same environment under a low-rank assumption. In addition to alleviating coverage requirements, so each task need not visit every state as long as others do, the method offers scalable evaluation of multiple tasks under new environments as computationally intensive planning scales with rank rather than the number of tasks. We provide finite sample guarantees on reward recovery and on policy learning in new environments. Experiments show our method is robust to limited coverage, recovers rewards on and off of each task's support, transfers to target environments at lower regret than baselines, with its computational advantage over per-task methods widening as tasks grow.

\end{abstract}
\section{Introduction}
\input{intro}

\section{Related Literature}
\input{related_lit}
\section{Method}\label{sec:method}
\input{method}

\section{Theoretical Guarantees}

\input{theory}

\section{Experiments}
\input{experiments}

\section{Conclusion}

\input{conclusion}

\section*{AI use statement}

In this work, we used generative AI tools for parts of coding up experiments (scaffolding for reproducibility from random seeds and tuning hyperparameters etc), figure generation and to flag mismatches in notation and typos in proofs and theorems remaining from multiple authors and the drafting process. We have reviewed all AI-assisted work: every proof, claim, and passage was checked by the authors, and we take responsibility for the final content of this work, including text, claims, and artifacts produced with the aid of generative AI.

\section*{Reproducibility Statement}
Assumptions and full proofs of Lemmas and Theorem 1 are in Appendix~\ref{sec:proofs}, and all experimental details (environments, reward construction, expert training, baselines, architectures, and the held-out tuning protocol) are in Appendix~\ref{app:experiments}. All environments are synthetic with known ground-truth rewards; every number is averaged over 10 seeds (mean and standard deviation), and anonymized code is included in the supplementary material.

\bibliography{irl}
\bibliographystyle{iclr2027_conference}

\appendix

\input{appendix}

\end{document}

%% file: intro.tex
Inverse reinforcement learning aims to recover rewards that rationalize behavior in expert demonstrations \citep{russell1998learning}. In that sense rewards are an agent's revealed preferences \citep{ziebart2008maximum}, enabling imitation learning \citep{abbeel2004apprenticeship, ho2016generative}, and support alignment from human feedback \citep{christiano2017deep}. 
Rewards are intrinsic to the agent's motives, as opposed to policies which are the result of planning given both the agent's rewards and the dynamics of the environment \citep{lucas1976econometric}. Hence we can use rewards to re-plan in perturbed environments and predict counterfactual behavior when agents face a new environment.

Many settings, such as online platforms and transportation networks, feature heterogeneous tasks. Following \cite{ng2000algorithms}, we understand a task as mapping one-to-one to a (here agent-specific) reward function. Shutting down a train line, for instance, leads commuters with distinct preferences over speed, cost, and crowding to re-route accordingly. Therefore, a planner must anticipate this re-optimization despite limited state coverage by each commuter. We consider $N$ independent tasks, that is $N$ agent-specific reward functions, for each of which we observe the induced soft-optimal behavior in a source environment. We then seek to learn these rewards so as to anticipate how the behavior of agents would transfer to a target environment.




Single-task demonstrations typically concentrate on a narrow region of the state space, so partial coverage weakens reward identification \citep{lazzati2024offline}. In particular, in the maximum-entropy RL framework, this materializes as temperature varies: high temperature improves coverage but flattens choice probabilities, while low temperature sharpens the reward signal but narrows coverage \citep{haarnoja2017reinforcement}. Fitting $N$ tasks independently inherits the single-task support problem. Here, we break this temperature-indexed tradeoff by encoding an often-natural structural relatedness of tasks by a low-rank assumption. As we prove formally, this lessens per-agent coverage requirements.

Learning under a low rank assumption pools data from multiple source tasks to learn a common-across-tasks low-dimensional representation of the tasks' input, in our context, agents' preferences. The statistical gain can be understood as stemming from needing to only ensure good coverage in that compressed representation as opposed to in a fully unrestricted per-agent, per-state, and per-action sense. In the context of multi-task reward learning for planning in a target environment, we find that a natural-to-analyze setting is not low-rankness of rewards themselves but low-rankness of policy logits.
This comes about from a representation of solutions to the maximum-entropy IRL problem in which rewards are affine in log policies \citep{cao2021identifiability}. Shared dynamics and subgoals make such low-dimensional structure plausible \citep{barreto2017successor}, and placing the assumption on choice probabilities rather than latent rewards makes it empirically testable (Section~\ref{sec:rank}). 


The combination of low rank log policies and a reward operator that is affine in log policies implies that the typically expensive planning step required for IRL reduces from being per-task to per-basis function. $\rankbound$ basis fixed points can be linearly composed to get $N$ reward functions instead of $N$ fixed points for $N$ reward functions. When $\rankbound\ll N$, as can be expected in online settings with many users, this represents a substantial computational speed-up. 

Our contributions are threefold. First, we introduce an offline, feature-free multi-task IRL method that pools coverage while retaining distinct per-task rewards and requiring only $\rankbound$ basis reward-inversion solves. Second, we prove finite-sample guarantees for logit recovery, target-policy evaluation, and policy recovery under a coverage condition that permits missing task–state blocks. Third, a causal coverage intervention validates the predicted mechanism, and experiments across three domains show stronger off-support reward recovery, robust counterfactual transfer in RecSim, and wall-clock time growing as $N^{0.21}$, up to $13\times$ faster than per-task GenPQR at $N=128$.

%% file: related_lit.tex
\noindent\textbf{Single Task IRL} 
Max-entropy IRL fits rewards by likelihood \citep{ziebart2008maximum}, while scalable variants use adversarial occupancy matching or soft $Q$-functions \citep{fu2018learning,garg2021iqlearn}. Partial coverage limits offline reward identification \citep{lazzati2024offline}. We build on the classify-then-regress view of \citet{van2025inverse}; the equivalent dynamic-discrete-choice literature estimates conditional choice probabilities and heterogeneous agent types \citep{rust1987optimal,arcidiacono2011conditional}.

\noindent\textbf{Multi-task IRL} 
Continuous latent-context methods such as PEMIRL infer task embeddings without known reward features, while hierarchical adversarial approaches learn shared multi-task structure \citep{yu2019meta,chen2023multitask}. These methods require simulator interaction and adversarial optimization and provide neither our finite-sample guarantees under missing coverage nor rank-scaled reward inversion. Multi-intention methods recover $K_{MI}$ discrete reward clusters and assign demonstrators to a mixture over them, often clustering directly on known reward features \citep{babesvroman2011apprenticeship,choi2012nonparametric,ramponi2020truly}; we instead recover $N$ distinct rewards through continuous loadings on $\rankbound$ shared bases.
  
\noindent\textbf{Low rank rewards and matrix completion} 
The closest work imposes low rank on rewards. \citet{shehab2025efficient} recover one agent's time-varying reward by nuclear-norm minimization, with rank across time and a known tabular model; \citet{guo2026provable}, extending \citet{nagaraj2023multiuser}, learn a low-rank reward matrix across tasks in forward RL, from sparsely observed rewards with online exploration and per-task planning. We instead impose low rank on log-policies and recover rewards from offline behavior alone, without reward features and with per-basis planning. Our analysis builds on nuclear-norm matrix completion \citep{negahban2012restricted}, low-rank multinomial-logit choice models \citep{kallus2016dynamic}, and completion under arbitrary sampling \citep{foygel2011learning}. We differ in allowing unobserved task--state blocks through a restricted coverage constant, in needing no sample-size condition, and in carrying the completion error through the Bellman reward operator to target-policy KL.



%% file: method.tex
\subsection{Setup}\label{sec:setup}
We consider $N$ tasks, indexed by $i\in[N]:=\{1,\ldots,N\}$. We observe demonstrations in a source environment and seek to predict how each task re-optimizes under target dynamics. Let $\StateSpace$ and $\ActionSpace$ be finite state and action spaces where $S:=|\StateSpace|$ and $A:=|\ActionSpace|$. For each task $i$, define the source $(\ell=1)$ and target $(\ell=2)$ Markov decision processes (MDPs) by $M_{\ell,i} =
(\StateSpace,\ActionSpace,\Transition_\ell,r_i,\InitialDist,\gamma_\ell)$. Here $\Transition_\ell$ is the transition kernel, $r_i(s,a)$ is task $i$'s finite mean reward at state $s$ and action $a$, $d_0$ is the initial state distribution, and $0<\gamma_\ell <1$ is the discount factor. For each task, the source and target MDPs share $\StateSpace,\ActionSpace,r_i$, and $d_0$, while their transition kernels and discount factors may differ. 

We assume that for each task $i$, the source policy $\pi^\star_{1,i}$ is soft optimal in $M_{1,i}$, that is, it maximizes entropy-regularized discounted return from every state with entropy coefficient $1$ \citep{ziebart2008maximum}. For theoretical analysis, we assume that $\Transition_\ell$, $d_0$ and $\gamma_\ell$ are known. In practice, reward recovery uses observed source transitions through fitted Q iteration (FQI) and hence does not require a model of $\Transition_1$. We assume target environment is specified by the practitioner, and re-planning requires known target dynamics. Our method first fits a shared policy model without known reward features, then recovers mean rewards by solving for a small set of shared basis functions. 



\subsection{Multi-task policy learning}
We observe $n$ tuples from the source environment $(I_t,S_t, A_t, S_t')$, generated by drawing a task-state pair $(I_t, S_t)\sim d$, an expert action $A_t\mid (I_t,S_t)\sim \pi_{1,I_t}^\star(\cdot\mid S_t)$, and a next state $S_t'\sim \Transition_1(\cdot\mid S_t,A_t)$. We let $d$ denote the joint distribution over tasks and states in the source environments, with joint probability $d_{i,s}=\mathbb{P}(I_t=i,S_t=s)$ and $p_i:=\sum_s d_{i,s}>0$. Since rational agents optimize task-specific rewards, we expect that $d_{i,\cdot}$ will concentrate in some parts of the state space and hence may leave task-state pairs unobserved. 
We seek to learn rewards for all tasks on this offline dataset alone, without access to the true source transition kernel or reward features. 

Tasks share source dynamics and may share subgoals with varying importance across tasks. Let $\Theta\in \R^{N\times SA}$ be a matrix of source policy logits, with source policies  $\pi_i(a\mid s)=\softmax(\Theta_{i,(s,\cdot)})_a$.
We parametrize these logits with  $\rankbound$ basis functions $\Phi(s,a):= (\phi_1(s,a), \dots, \phi_{\rankbound}(s,a))^\top$ and task-specific coefficients $W_i\in \R^{\rankbound}$, so that $\Theta_{i,(s,a)}=W_i^\top \Phi(s,a)$. Stacking task coefficients and basis evaluations
gives $\Theta=W\Phi\in \R^{N\times SA}$ of rank at most $\rankbound$.\footnote{This is a structural approximation, not a consequence of linear rewards: the maximum-entropy optimality map is nonlinear and can increase effective rank.}
Because softmax is invariant to a common shift of action logits, we fix a reference action $a_0\in\ActionSpace$ and enforce $\Theta_{i,(s,a_0)}=0$. We implement this constraint by setting $\Phi(s,a_0)=0$ for all $s\in \StateSpace$.

Let $D_i:=\{(S_t,A_t,S_t'):I_t=i\}$ denote task $i$'s source transitions.
Let $\triplenorm{\cdot}_F$ be the Frobenius norm and $\lambda>0$ be the regularization strength. We estimate the factors jointly by a penalized maximum likelihood objective 
\begin{equation}
    \widehat W, \widehat \Phi \in \argmin_{W , \Phi } - \frac{1}{n} \sum_{t=1}^n \log \pi_{I_t}(A_t|S_t) + \frac{\lambda}{2} \triplenorm{W}_F^2+ \frac{\lambda}{2} \triplenorm{\Phi}_F^2 
    \label{eq:low_rank_obj_fn}
\end{equation}
The fitted logit matrix is given by $\widehat\Theta:=\widehat{W}\widehat{\Phi}$, resulting in a fitted log-policy matrix, $\widehat u\in \R^{N\times SA}$ where
$\widehat{u}_{i,(s,a)} := \widehat\Theta_{i,(s,a)} - \mathrm{logsumexp}(\widehat{\Theta}_{i,(s,\cdot)})$. Sharing the basis allows tasks to pool information across their observed states. Individual tasks therefore need not be observed in every state, provided the pooled observations distinguish errors compatible with the low-rank structure. We make this coverage condition precise in Assumption~\ref{ass:curvature}, similar in spirit to task-diversity conditions in multi-task representation learning \citep{tripuraneni2020theory}. Lemma~\ref{lem:unweighted-coverage} gives a sufficient spectral condition.




\subsection{Efficient reward recovery through shared bases}\label{sec:reward_recovery}

In the max-entropy setting, different reward functions can induce the same source policy~\citep{ng1999policy,cao2021identifiability}. To identify the reward of each task, we assume that its true reward, shared across the source and target environments, satisfies the following normalization. 

\begin{assumption}
\label{ass:normalization}
For the fixed known $\mu:\StateSpace\to\Delta(\ActionSpace)$ and $g:\StateSpace\to\R$ used in the procedure, every task's true reward satisfies
\begin{equation}
    \sum_{a\in\mathcal A}\mu(a\mid s)r_i(s,a)=g(s)
    \qquad\forall i\in [N],s\in\mathcal S.
    \label{eq:reward-normalization}
\end{equation}
\end{assumption}

Following the setup in \citet{van2025inverse}, write $(\Transition_1 V)(s,a):=\sum_{s'}\Transition_1(s'\mid s,a)V(s')$ and $(\mu f)(s):=\sum_b\mu(b\mid s)f(s,b)$, and define the operators $(\Transition_1\mu\, f)(s,a) := \sum_{s',b}\Transition_1(s'\mid s,a)\,\mu(b\mid s')\,f(s',b)$ and 
$(C_\mu f)(s,a) := f(s,a)-\sum_b\mu(b\mid s)\,f(s,b)$.
\begin{equation}
L_\mu := C_\mu(I-\gamma_1 \Transition_1\mu)^{-1} \label{eq:L-mu}
\end{equation}
Note that the inverse exists because $\Transition_1\mu$ is row-stochastic and $\gamma_1<1$.
\begin{lemma}[Normalized reward is affine in the log-policy]\label{lem:linearity}
Let $u_i^\star(s,a):=\log\pi_{1, i}^\star(a\mid s)$. The unique reward satisfying the source soft Bellman equations and normalization~\eqref{eq:reward-normalization} is
\begin{equation}
r_i = L_\mu(u_i^\star-g)+g.
\label{eq:reward-map}
\end{equation}
\end{lemma}
Since $L_\mu$ is bounded and linear, $r_i$ is affine in $u_i^\star$.
Appendix Lemma~\ref{dn:lm:irl_reward} proves uniqueness.
The statistical theory below analyzes the exact normalized plug-in reward. Given fitted anchored logits
\[
\widehat \Theta_{i,(\cdot,\cdot)}=\sum_{j=1}^{\rankbound}\widehat W_{ij}\widehat\phi_j,
\qquad
\widehat c_i(s):=\log\sum_a\exp\{\widehat \Theta_{i,(s,a)}\},
\]
the fitted log-policy is $\widehat u_i=\widehat \Theta_{i}-\widehat c_i$, where
$\widehat c_i$ is broadcast across actions. Its exact plug-in reward is therefore
$
\widehat r_i^{\mathrm{plug}}
=L_\mu(\widehat \Theta_i-\widehat c_i-g)+g.
$
The normalizer $\widehat c_i$ depends nonlinearly on the task loadings, so the
matrix of normalized log-policies need not remain rank $\rankbound$. Applying the reward operator exactly can consequently require one inversion per task, forfeiting the principal computational benefit of the low-rank representation.

Our implementation makes a deliberate basis-planning approximation: it applies
$L_\mu$ directly to the low-rank anchored logits $\widehat r_i^{\mathrm{approx}}
:=L_\mu(\widehat \Theta_i-g)+g
=\sum_{j=1}^{\rankbound}\widehat W_{ij}L_\mu\widehat\phi_j+(I-L_\mu)g$. This requires only $\rankbound$ basis reward-inversion solves, rather than $N$ task-specific solves. The approximation
gap is explicit $\widehat r_i^{\mathrm{approx}}-\widehat r_i^{\mathrm{plug}}
=L_\mu\widehat c_i$, which need not vanish under action-dependent transitions. 
Thus the theory isolates the statistical error of the normalized plug-in
estimator, while the practical method adds this separately characterized
computational approximation. All reported experiments use
$\widehat r_i^{\mathrm{approx}}$ and demonstrate that the shortcut retains strong off-support recovery and transfer performance while delivering the intended
rank-scaled computation. Appendix~\ref{app:raw_logit_xp} directly compares reward recovery between the exact method and its approximation and finds mild to null effects, with wall-clock savings that grow with $N$.

For a basis function $\phi$ and a state-action function $v$, define the affine Bellman operator
\begin{equation}
(T_\phi v)(s,a) := \big(\Transition_1\mu\big)(\gamma_1 v - \phi)(s,a)
= \mathbb{E}_{s'\sim \mathsf P_1(\cdot\mid s,a)}\!\Big[\textstyle\sum_b \mu(b\mid s')\big(\gamma_1 v(s',b) - \phi(s',b)\big)\Big].
\label{eq:basis-bellman}
\end{equation}
Since $\mathsf P_1\mu$ is stochastic and $\gamma_1<1$, $T_\phi$ is a $\gamma_1$-contraction with a unique fixed point $v_\phi$, and the basis reward of Lemma~\ref{lem:linearity} is recovered as $L_\mu\phi = C_\mu(\phi - \gamma_1 v_\phi)$. We compute $v_\phi = T_\phi v_\phi$ by fitted iteration with a value network, regressing at each step onto the targets $T_\phi v$ evaluated at the observed next-states; we use the same neural evaluation across all environments.

\subsection{Algorithm}
\label{sec:algorithm}
\begin{algorithm}[t]
\caption{Low-Rank Multi-Task IRL}
\label{alg:main}
\begin{algorithmic}[1]
\REQUIRE Demonstrations $D_{\mathrm{total}}=\{D_i\}_{i=1}^N$, rank $\rankbound$, discount $\gamma_1$, reward normalization $(\mu, g)$.  
\STATE \textbf{Initialize:} $\rankbound$ basis networks for $v_{\phi_j}$ and $v_{g}$
\STATE \textbf{Classify:} Solve the joint softmax MLE \eqref{eq:low_rank_obj_fn} on $D_{\mathrm{total}}$ $\to\; \hat
W, \hat\Phi$
\FOR{$j = 1,\ldots,\rankbound$}
\STATE \textbf{Regress:} Solve $\hat v_{\hat\phi_j} = T_{\hat\phi_j}\hat v_{\hat\phi_j}$ via fitted iteration on $D_{\mathrm{total}}$
\STATE \textbf{Basis reward:} $\hat r_{\hat\phi_j} \leftarrow \hat\phi_j - \gamma_1\hat v_{\hat\phi_j} + \mu(\gamma_1\hat v_{\hat\phi_j} - \hat\phi_j)$
\ENDFOR
\STATE \label{alg1:offset} \textbf{Offset:} Solve $\hat v_g = T_g\hat v_g$ via fitted iteration on $D_{\mathrm{total}}$; set $\ \hat b_g \leftarrow g + \gamma_1(\hat v_g - \mu\hat v_g)$ 
\FOR{$i = 1,\ldots,N$}
\STATE $\hat r_i \leftarrow \sum_{j=1}^{\rankbound} \hat W_{ij}\,\hat r_{\hat\phi_j} + \hat b_g$ \hfill $\triangleright$ closed-form linear combination
\ENDFOR
\RETURN $\{\hat r_i\}_{i=1}^N$
\end{algorithmic}
\end{algorithm}
Algorithm~\ref{alg:main} has two stages: jointly fit rank-$\rankbound$ anchored logits, and then run $\rankbound$ FQI regressions on the learned basis functions and linearly recombine their rewards. A new task adds only $\rankbound$ task-specific loadings and additional classification data; it does not add a reward-inversion solve, so the expensive stage scales with $\rankbound$ not $N$. Section~\ref{sec:comp_efficiency} verifies this computational advantage empirically. 

For continuous-state environments, each basis function is represented by a neural network mapping a state to $\mathbb{R} ^{|\ActionSpace|}$ logits, with task logits formed as $\widehat W_i^\top\widehat\Phi(s,\cdot)$ before the softmax. Architecture, optimization, and regularization details are provided in Appendix~\ref{sec:app-tuning}.

Note that if $g=0$ then we can skip line \ref{alg1:offset} of the algorithm and simply set $\hat b_g=0$.

%% file: theory.tex
We first bound the estimation error of the shared source logit matrix under a pooled source coverage condition, and then show simultaneous bounds on target-policy value error, regret, and occupancy-weighted KL divergence from each task's soft-optimal target policy. Appendix~\ref{sec:proofs} gives the full details.

\subsection{Statistical assumptions and estimator}\label{sec:theory-setting}
Following the setting of Section~\ref{sec:method}, let $\Theta^\star$ be the anchored logit matrix of the true source policies, given by $\Theta^\star_{i,(s,a)} :=\log\pi_{1,i}^\star(a\mid s)-\log(\pi_{1,i}^\star(a_0\mid s))$.
For analysis, the observations are i.i.d., and the source design $d$ is known.

\begin{assumption}
\label{ass:low-rank}
Assume that the true source logit matrix $\Theta^\star$ satisfies $\Theta^\star\neq 0$\footnote{The uniform-policy case $\Theta^\star=0$ is treated separately in Proposition~\ref{dn:prop:zero-truth} in Appendix~\ref{dn:app:zero-truth}.}, $\rank(\Theta^\star)\leq \rankbound$ and $\norm{\Theta^\star}_\infty\leq B$\footnote{$\norm{\Theta}_\infty:=\max_{i,s,a}|\Theta_{i,(s,a)}|$ denotes the entrywise maximum norm.} where $\rankbound \leq \min \{N, SA\}$ and $B>0$ is fixed independent of $N, S$ and $A$.
\end{assumption}

We analyze the bounded convex nuclear-norm formulation of the likelihood objective~\eqref{eq:low_rank_obj_fn}. For observation $t$, the negative log likelihood is $l_t(\Theta)=\log\left(\sum_{a\in\ActionSpace}e^{\Theta_{I_t,(S_t,a)}}\right)-\Theta_{I_t,(S_t,A_t)}$, and the empirical loss is $L_n(\Theta):=\frac1n\sum_{t=1}^n l_t(\Theta)$.
Fix a failure probability $0<\delta<1$, define $\nu 
:=\max\left\{
\left(1-\frac1A\right)\max_i p_i,\;
\frac12\max_s\sum_i d_{i,s}
\right\}$ and set the regularization parameter to
\begin{equation}
\lambda_n
:=2\left\{
\sqrt{\frac{2\nu }{n}\log\frac{2(N+SA)}{\delta}}
+\frac{2\sqrt2}{3n}\log\frac{2(N+SA)}{\delta}
\right\}.
\label{eq:unweighted-lambda}
\end{equation}
We estimate the source logit matrix by the nuclear norm regularized maximum likelihood
\begin{equation}
\widehat\Theta \in\argmin_{\Theta\in\mathcal M_B} \left\{L_n(\Theta)+\lambda_{n}
    \triplenorm{\Theta}_*\right\}\footnote{Here $\triplenorm{\cdot}_*$ is the nuclear norm. For $\Theta\in\R^{N\times SA}$ and $\rankbound\geq\rank(\Theta)$, $\triplenorm{\Theta}_*=\min_{W\Phi=\Theta}\frac12(\triplenorm{W}_F^2+\triplenorm{\Phi}_F^2)$, where the minimum is over $W\in\R^{N\times\rankbound}$ and $\Phi\in\R^{\rankbound\times SA}$. Thus the Frobenius penalties in~\eqref{eq:low_rank_obj_fn} give a variational form of nuclear-norm regularization.},
    \label{eq:unweighted-estimator}
\end{equation}
where the feasible set is $\mathcal M_B:=\left\{\Theta\in\R^{N\times SA}:
    \norm{\Theta}_\infty\leq B,
    \Theta_{i,(s,a_0)}=0\ \text{for every }i,s\right\}$.

We set $\widehat u_{i,(s,a)}
:=\widehat\Theta_{i,(s,a)}
-\log\sum_{b\in\ActionSpace}
\exp\{\widehat\Theta_{i,(s,b)}\}$ and recover $\widehat r_i:=L_\mu(\widehat u_i-g)+g$ via~\eqref{eq:reward-map}.

\subsection{Pooled source coverage}
Let $q:=\rank(\Theta^\star)\leq\rankbound$ and let $\Theta^\star=U\Sigma V^\top$ be a singular value decomposition. Its tangent space is $T:=\{UX^\top+YV^\top:X\in\R^{SA\times q},\ Y\in\R^{N\times q}\}$. Write $P_T$ and $P_{T^\perp}$ for the Frobenius-orthogonal projections onto $T$ and its complement. The regularized estimator's error lies, with high probability, in the anchored cone
\begin{equation}
\mathcal C(\Theta^\star):=\left\{\Delta:
\Delta_{i,(s,a_0)}=0\ \forall i,s,\quad
\triplenorm{P_{T^\perp}\Delta}_*\leq3\triplenorm{P_T\Delta}_*\right\}.
\label{eq:unweighted-cone}
\end{equation}
We measure source coverage on these possible errors by
\begin{equation}
c_0(d,\Theta^\star):=
\inf_{\Delta\in\mathcal C(\Theta^\star)\setminus\{0\}}
\frac{\sum_{i,s}d_{i,s}\norm{\Delta_{i,(s,\cdot)}}_2^2}
{\triplenorm{\Delta}_F^2}.
\label{eq:unweighted-restricted-coverage}
\end{equation}
\begin{assumption}
\label{ass:curvature}
$c_0(d,\Theta^\star)>0$.
\end{assumption}
Intuitively, this condition holds exactly when every nonzero error in the cone changes a task--state block with $d_{i,s}>0$ (Lemma~\ref{dn:lem:general-coverage-nullspace} in Appendix~\ref{dn:app:coverage}). This is a population coverage condition determined by the source design and the shared logit geometry. In task--state--action order, define $\mathbf D_d:=\operatorname{diag}((d_{i,s})_{i,s,a})$ and define $\mathbf P_T\in\R^{NSA\times NSA}$ as the unique matrix satisfying $\mathbf P_T\operatorname{vec}(Z)=\operatorname{vec}(P_TZ)$ for every $Z\in\R^{N\times SA}$.
\begin{lemma}
\label{lem:unweighted-coverage}
Under Assumption~\ref{ass:low-rank}, let $\lambda_{\min}$ denote the smallest eigenvalue, then
\[{c_0(d,\Theta^\star)\geq
\sup_{\tau\geq0}\left\{
\lambda_{\min}(\mathbf D_d-\tau\mathbf P_T)
+\frac{\tau}{1+18\rankbound}\right\}
\geq\min_{i,s}d_{i,s}}.
\]
\end{lemma}
Appendix Lemma~\ref{dn:lem:unweighted-coverage} proves this bound. At minimum, the tasks visiting each state must collectively span the latent directions; otherwise $c_0=0$. Example~\ref{dn:ex:missing-block} constructs an instance with unobserved task--state block yet positive $c_0$, permitting full reward recovery and vanishing target regret as $n\rightarrow \infty$, even when target policies visit the unobserved source block (Theorem~\ref{thm:unweighted-regret}). Section~\ref{sec:coverage_xp} tests this directional coverage mechanism empirically.


\subsection{Target values and finite-sample guarantees}
Let $\Pi$ be the set of all stationary target policies. 
For $\pi\in\Pi$, denote $d_2^\pi(s):=(1-\gamma_2)\E_{\mathsf P_2,\pi,d_0} \left[\sum_{t\geq0}\gamma_2^t\one\{S_t=s\}\right]$ the normalized discounted target state occupancy, and $\rho_2^\pi(s,a):=d_2^\pi(s)\pi(a\mid s)$ the normalized discounted target state-action occupancy.\footnote{Both $d_2^\pi$ and $\rho_2^\pi$ are known because the target dynamics and initial distribution are known.} 

For each candidate $\pi\in\Pi$, we estimate its target soft value using the recovered reward $\widehat r_i$:
\begin{equation}
\widehat\softvalue_{2,i}(\pi)
:=\softvalue_{2,i}(\pi;\widehat r_i)
=\frac1{1-\gamma_2}\sum_{s,a}\rho_2^\pi(s,a)
\{\widehat r_i(s,a)-\log\pi(a\mid s)\},
\label{eq:target-soft-value}
\end{equation}
The true value $\softvalue_{2,i}(\pi;r_i)$ uses $r_i$ in place of $\widehat r_i$. For each task $i$, maximize the estimated target soft value:
\begin{equation}
\widehat\pi_i\in\argmax_{\pi\in\Pi}\widehat\softvalue_{2,i}(\pi).
\label{eq:plugin-selector}
\end{equation}




Let $\pi_{2,i}^*$ be the soft-optimal policy for $M_{2,i}$, characterized by its soft Bellman equations. In particular, $\pi_{2,i}^* \in \arg\max_{\pi'}\softvalue_{2,i}(\pi'; r_i)$. The logit-error radius $\varepsilon_n$ and target-value radius $R_n(\pi)$ satisfy
\begin{align*}
\varepsilon_n &\lesssim
\frac{\sqrt{\rankbound\nu\log(2(N+SA)/\delta)/n}+\sqrt{\rankbound}\,\log(2(N+SA)/\delta)/n}
{c_0(d,\Theta^\star)},\\
R_n(\pi) &\lesssim
\frac{\varepsilon_n}{1-\gamma_2}\norm{L_\mu^\top\rho_2^\pi}_2.
\end{align*}
Here $\lesssim$ hides constants depending only on $A$ and $B$. We give their exact definitions in~\eqref{dn:eq:unweighted-logit-radius} and~\eqref{dn:eq:unweighted-value-radius}. 

\begin{theorem}
\label{thm:unweighted-regret}
Under the setting of Section~\ref{sec:theory-setting} and Assumptions~\ref{ass:normalization},~\ref{ass:low-rank}, and~\ref{ass:curvature}, use the estimator~\eqref{eq:unweighted-estimator} with tuning~\eqref{eq:unweighted-lambda}, exact plug-in reward recovery from Section~\ref{sec:reward_recovery}, and the selector~\eqref{eq:plugin-selector}. With probability at least $1-\delta$, the following hold simultaneously for every logit minimizer, every task $i$, every $\pi\in\Pi$, and every maximizing choice $\widehat\pi_i$.

\emph{(i) Estimation and value:} $\triplenorm{\widehat\Theta-\Theta^\star}_F\leq\varepsilon_n$, and $\left|\widehat\softvalue_{2,i}(\pi)-\softvalue_{2,i}(\pi;r_i)\right|
\leq R_n(\pi)$.

\emph{(ii) Target regret:} $\softvalue_{2,i}(\pi_{2,i}^\star;r_i)-\softvalue_{2,i}(\widehat\pi_i;r_i)
\leq R_n(\pi_{2,i}^\star)+R_n(\widehat\pi_i)$.

\emph{(iii) Policy error:} $\sum_s d_2^{\widehat\pi_i}(s)\,
\operatorname{KL}\!\left(\widehat\pi_i(\cdot\mid s)\,\middle\|\,
\pi_{2,i}^\star(\cdot\mid s)\right)
\leq(1-\gamma_2)\{R_n(\pi_{2,i}^\star)+R_n(\widehat\pi_i)\}$.
\end{theorem}

The KL bound controls action discrepancies at states visited by the learned policy, weighted by its discounted occupancy. When
$n\nu\geq\log(2(N+SA)/\delta)$, it simplifies up to constants to $\frac{1}{c_0(d,\Theta^\star)}\,
\bigl(\|L_\mu^\top\rho_2^{\pi_{2,i}^\star}\|_2 + \|L_\mu^\top\rho_2^{\widehat\pi_i}\|_2\bigr)\,\sqrt{\frac{\rankbound \nu}{n}\log\frac{2(N+SA)}{\delta}}$.

The bound separates three effects: $c_0^{-1}$ penalizes latent directions weakly identified by the source demonstrations, the norm terms capture how reward recovery and target visitation amplify source logit estimation error, and the leading statistical term is of order $\sqrt{\frac{\rankbound\nu}{n}\log\frac{2(N+SA)}{\delta}}$.

%% file: experiments.tex
\label{sec:experiments}
We evaluate whether LowRank (i) exploits pooled coverage, (ii) recovers rewards sample-efficiently on and off task support, (iii) transfers under target dynamics shift, and (iv) exhibits sublinear wall-clock growth as the task population increases fixing basis rank. The three domains are FourRooms, a partially observed sparse-reward navigation task; Highway, a continuous-state driving task; and RecSim, a tabular recommender with nonlinear topic fatigue. Rewards use four hidden task-preference dimensions; full environment and expert-training details appear in Appendix~\ref{app:experiments}.

We report Pearson correlation with $\mu$-normalized ground-truth rewards on and off-support evaluation states. For continuous domains, these sets are defined over the pooled empirical evaluation states: a state is on-support if observed for that task and off-support if observed only for other tasks. This does not claim recovery at arbitrary unseen points in the continuous state
space.

Baselines are per-task and pooled GenPQR and IQ-Learn, plus a feature-free four-cluster Multi-intent adaptation. Continuous latent meta-IRL and hierarchical adversarial methods require interactive simulator access and are therefore outside
our offline setting. All applicable methods share architectures and held-out tuning; Appendix~\ref{sec:app-tuning} gives details.



\subsection{Coverage drives off-support recovery}\label{sec:coverage_xp}
Our coverage condition is directional rather than a visit count: at state $s$, error along latent direction $v_k(s)$ enters the training loss only through the $k$-th eigenvalue $\lambda_k(s)$ of the visiting tasks' Gram matrix $\sum_i d_{i,s}(U^\star_{i,:})^\top U^\star_{i,:}$, so errors along directions the visiting tasks do not span, never appear in the training loss yet persist off-support; a singular Gram at any $s$ forces $c_0=0$ (Appendix~\ref{app:coverage}).

We test this mechanism by progressively deleting observations from the tasks
contributing most to selected states' leading eigendirection. Matched controls
lose the same number of observations without targeting task alignment, while
placebo states lose none. After refitting, the left panel reports the change
from zero deletion in absolute error projected onto the targeted direction;
the right reports treated-minus-control difference-in-differences for the
targeted versus remaining $\rankbound-1$ directions.

\begin{figure}[t!]
\centering
\includegraphics[width=\linewidth]{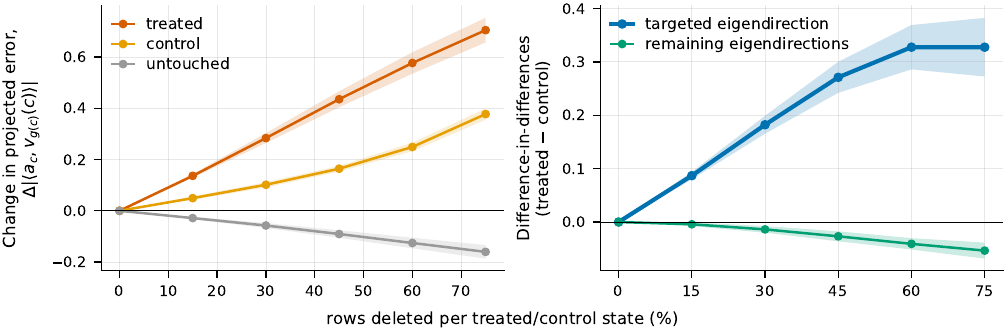}
\caption{Causal test of directional coverage. Left: change in error projected
onto the targeted eigendirection for treated, volume-matched control, and
placebo cells. Right: treated-minus-control difference-in-differences for the
targeted and remaining directions.}
\label{fig:dosage_coverage}
\end{figure}

Targeted projected error grows monotonically with deletion dose, and the
difference-in-differences is concentrated in that direction. This isolates
loss of directional coverage, rather than sample volume, as the source of the
additional off-support error.



\subsection{Sample efficient reward recovery}\label{sec:sample_efficiency}
We vary the number of decisions per task from $100$ to $10{,}000$ while holding
the task population and model architecture fixed. Figure~\ref{fig:data_ablation}
reports reward correlation separately on states observed for each task (on-support,
dotted) and states supplied only by other tasks (off-support, solid). This isolates
whether additional within-task data can substitute for coverage pooled across
tasks.

LowRank achieves the highest off-support correlation at every data budget in all three
domains, with the largest gains in the low-data regime. Per-task methods improve
substantially with $n$ and approach LowRank on Highway at $10{,}000$ decisions,
where sufficient within-task data eventually supplies broad coverage. The same
gap persists in FourRooms and RecSim, whose partial observability and fatigue
dynamics leave structured support gaps even at larger budgets. Naively pooled
methods remain poor because they gain coverage only by discarding task
heterogeneity, while discrete Multi-intent sharing remains less expressive than
LowRank's continuous task loadings.
\begin{figure}[t!]
\centering
\includegraphics[width=1\textwidth]{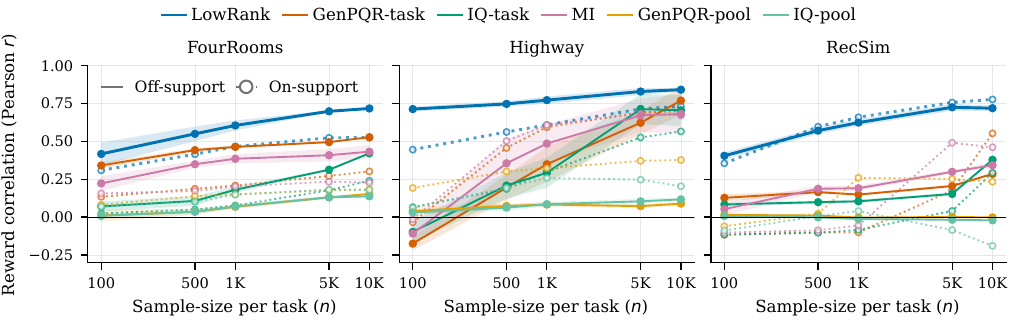}
\caption{Reward recovery versus samples per task (mean and standard deviation
over 10 seeds). LowRank leads off-support recovery at every budget; per-task methods narrow the gap only
on Highway at high sample sizes.}
\label{fig:data_ablation}
\end{figure}

\subsection{Policy recovery in perturbed environments}\label{sec:transfer_xp}
We evaluate counterfactual policy transfer in RecSim, where fatigue concentrates
source behavior and target perturbations naturally push occupancy into rarely
visited states. In our FourRooms and Highway configurations, ordinary trajectories
already expose the reward-relevant features, so plausible dynamics changes did not
produce a comparably clean off-support shift. Evaluating target return there would
also conflate reward-transfer error with approximate neural planning, whereas
RecSim's tabular dynamics permit exact target planning and oracle comparison.
We therefore use RecSim to isolate the question addressed by our theory: whether
recovered rewards remain valid when target occupancy moves beyond a task's source
support. We increasingly suppress each task's preferred genre in target slates and
compare the resulting policies with the oracle optimum.
  
\begin{figure}[t!]
\centering
\includegraphics[width=1\textwidth]{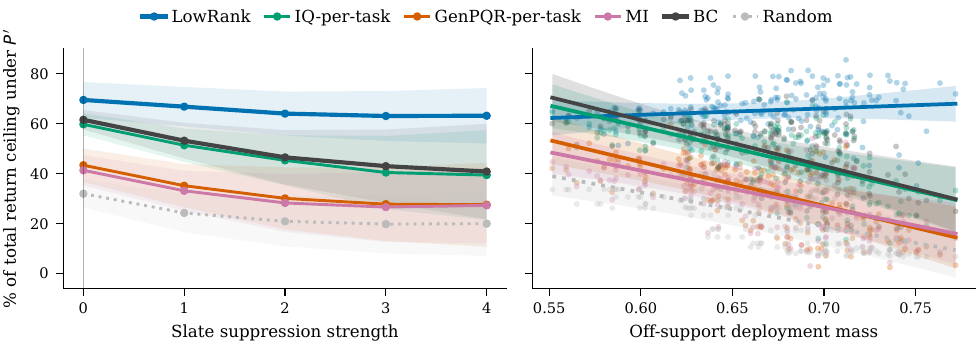}
\caption{Policy recovery in target perturbed environments. Error bands are $\pm$1 s.d over ten seeds.}
\label{fig:transfer}
\end{figure}

Figure~\ref{fig:transfer} shows that LowRank achieves the highest normalized
target return across perturbations. Behavior cloning and IQ-Learn are competitive
at zero shift, but behavior cloning deteriorates as the target diverges from the
source. In the right panel, LowRank's return is flat in off-support occupancy (slope $\hat\beta = +2.6$ percentage points per 0.1 off-support mass, 95\% CI $[-3.2, 8.3]$), whereas IQ-Learn and GenPQR fall by $17.1$ and $17.6$ points per $0.1$ of off-support mass. By Lemma~\ref{lem:kl_regret_equiv}, this return regret is equivalent to
occupancy-weighted policy KL.

\subsection{Computational Efficiency}\label{sec:comp_efficiency}
At fixed basis rank $\rankbound$, LowRank should depend only weakly on the task population because its expensive reward-inversion step uses $\rankbound$ rather than $N$ fitted evaluations. We test this by increasing $N$ from 8 to 128 with randomly sampled task coefficients. Figure~\ref{fig:scaling} reports wall-clock time on log--log axes: the per-task baselines scale approximately linearly, whereas LowRank grows as $N^{0.21\pm0.025}$ across environments; Multi-intent is also sublinear but noisier at $N^{0.34\pm0.20}$. Relative to per-task GenPQR, which uses the same reward operator, LowRank is roughly $13\times$, $6\times$, and $4\times$ faster at $N=128$ on Highway, FourRooms, and RecSim. IQ-Learn's cheaper per-task update makes it faster at small $N$, but it is overtaken by $N=128$; at that scale no baseline is both faster and more accurate than LowRank (Table~\ref{tab:scaling-quality}). Accuracy does not degrade with scale: at $N=128$ ($n=500$), LowRank's off-support reward correlation is $0.88$, $0.53$, and $0.78$ on Highway, FourRooms, and RecSim, versus $0.74$, $0.39$, and $0.18$ for per-task GenPQR. Multi-intent matches LowRank only on Highway ($0.88$) and trails on FourRooms and RecSim ($0.24$, $0.35$; Table~\ref{tab:scaling-quality}). 

\begin{figure}[t!]
\centering
\includegraphics[width=1\textwidth]{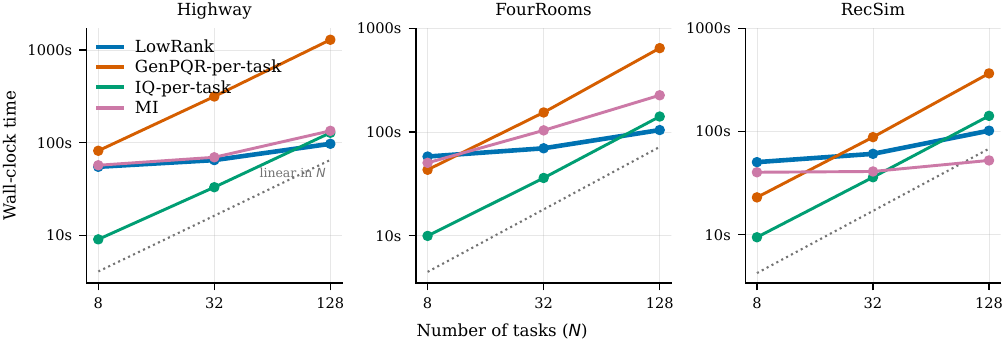}
\caption{Wall-clock scaling with task population $N$ at fixed basis rank $\rankbound$. LowRank exhibits substantially sublinear growth, while per-task methods scale approximately linearly}
\label{fig:scaling}
\end{figure}

%% file: conclusion.tex
We develop a multi-task IRL method that assumes a low-rank structure on anchored source logits. This transforms multi-task IRL from $N$ separate estimation problems into one shared problem, where identifiability does not require full  per-task coverage of states. 
Instead, we require a mild restricted-eigenvalue condition on the source-data design rather than conditions on density ratios: demonstrations must span the low-rank cone. 
At a minimum, the tasks visiting each state must span the latent task directions and each task must visit states spanning the latent basis directions. Experiments demonstrate that pooling can recover the correct direction of reward variation where single-task estimates reverse it outside of their support. Pooling also reduces planning from per-task to per-basis, yielding sub-linear wall clock scaling in the number of tasks.

Our work suggests some avenues of further investigation. Our method and experiments used generic function approximation, but our finite-sample guarantees were limited to tabular MDPs with a known source kernel. Extending the theory to cover this approximation error, and the log-partition residual of the rank-$\rankbound$ approximation, is a natural next step, as is relaxing the assumption of a shared reward normalization. Additionally, our environments are synthetic so that ground-truth rewards are available for clear comparisons between methods. Validation on real behavioral data remains open.

%% file: appendix.tex
\input{theory_appendix}

\section{Experiments}\label{app:experiments}

\subsection{Environments}
\label{app:environments}

All three environments share the linear-reward structure $r_i(s,a)=W_i^\top\psi(s,a)$
with $K=4$ features $\psi$, and the feature map $\psi$ is unavailable to every
learning method.

\paragraph{FourRooms.} A $19\times19$ grid of four rooms joined by single-cell
doorways, with a fixed wall layout, implemented in NAVIX \citep{pignatelli2025navix},
a JAX reimplementation of MiniGrid \citep{chevalier2023minigrid}.
Colored objects of $K=4$ types (12 per color, 48 total) are scattered on reachable
cells; the agent auto-collects an object by entering its cell, and collected objects
respawn at random empty cells. The per-step feature $\psi\in\{0,1\}^4$ is the color
one-hot of any object collected that step (zero otherwise), so rewards are sparse.
Observability is local: the agent sees an $11\times11\times4$ object window centered
on its position plus its normalized coordinates, and we stack the last three
(observation, action) frames for trajectory-aware classification. There are four
cardinal navigation actions and episodes last 100 steps. We use $N=8$ tasks (four
pure color-seekers and four mixtures). Expert policies are trained by PPO with a CNN
encoder.

\paragraph{Highway.} A multi-lane driving environment \citep{rachum2026bxrl} with
continuous kinematic state. The observation is the five nearest vehicles, each with
five features (presence, longitudinal and lateral position, and velocities), for a
$25$-dimensional ego-centric state; the action set is five discrete meta-actions
(change lane left/right, idle, accelerate, decelerate). The $K=4$ features are speed,
collision, headway (proximity to the lead vehicle), and lane change, and tasks differ
in their signed weights over these (e.g., preferring speed, tailgating, or penalizing
lane changes). NPC traffic is stochastic. We use $N=5$ tasks, with expert policies
trained by PPO with an MLP encoder.

\paragraph{RecSim.} A tabular sequential recommender \citep{ie2019recsim} over $T=4$
topics ($K=4$ features). The state is the vector of cumulative per-topic click counts,
which is monotone within an episode and takes $9^4=6561$ discrete values. At each of
the steps the user is shown 10 candidate items plus a skip action (11 actions in
total) and chooses via a forward-looking soft-$Q$ policy at temperature $0.1$. Each
topic has a burnout threshold (we use $[5,6,7,8]$); once a topic's cumulative clicks
reach its threshold, its appeal drops to zero for the remainder of the episode. The
feature $\psi$ is the clicked item's quality times its topic one-hot times the current
topic appeal. We use $N=8$ users (four pure and four 50/50 mixtures). Because the
state space is small, expert policies are computed exactly by tabular soft value
iteration, one per (user, discount) pair; this makes RecSim the setting where we can
evaluate transfer against oracle policies without approximation error
(Section~\ref{sec:transfer_xp}).

\subsection{Reward Construction and Expert Training}
\label{app:rewards}

For each environment we set task weights $W\in\mathbb R^{N\times 4}$ as a combination
of pure tasks (weight on a single feature) and mixed tasks (weight spread across
features), chosen so that the rows span $\mathbb R^4$. Choosing four reward features
does not force the rank $\rankbound$ of the log-policy matrix to be four; the forward soft-RL
map is nonlinear, so we use $\rankbound=4$ throughout and probe misspecification in Section~\ref{sec:rank}. Given each task's reward, we train an expert (PPO for
FourRooms and Highway, exact soft value iteration for RecSim), roll out trajectories
for offline training, and $\mu$-normalize the ground-truth rewards to match our
identification constraint (with $\mu$ uniform and $g=0$).

\subsection{Baselines}
\label{app:baselines}

We compare against two single-agent methods and one multi-task method.

\paragraph{GenPQR \citep{van2025inverse}.} The single-agent classify-then-regress
method our approach extends: a softmax classifier estimates the log-policy, which the
reward operator maps to a reward. We run it in two multi-task variants: \emph{per-task}
(fit independently on each task's data) and \emph{pooled} (fit a single reward on all
tasks' data, ignoring task identity).

\paragraph{IQ-Learn \citep{garg2021iqlearn}.} A non-adversarial occupancy-matching
method that learns a soft $Q$-function directly, from which a reward is extracted by
the inverse soft-Bellman operator. Run in the same \emph{per-task} and \emph{pooled}
variants as GenPQR.

\paragraph{Multi-intent IRL.} A feature-free, model-free adaptation of
\citet{babesvroman2011apprenticeship}. The original clusters demonstrations using
known reward features and a model-based inner loop, neither of which is available in
our setting. We keep the core structure: a soft-clustering EM over log-policies with
$K_{MI}=4$ clusters, where each cluster's reward is recovered from the data of the tasks
assigned to it (the same model-free operator as our method) and planning is performed
once per cluster. 
Each task's reward is its responsibility-weighted combination of the $K_{\mathrm{MI}}$ cluster rewards.

\subsection{Architectures, hyperparameters, and tuning}
\label{sec:app-tuning}

Within an environment all methods share the same encoder and width, so differences
reflect the estimation strategy rather than model capacity: a CNN over the local object view for FourRooms, an MLP for Highway, and an MLP over the tabular click state for RecSim. The classifier is trained with Adam.

Regularization differs by method, and each was given a matched opportunity to be tuned. Our classifier uses the variational nuclear-norm penalty with a scalar $\lambda_c/\sqrt{n}$, so the penalty weakens as the per-agent budget grows; the schedule form follows the penalised-MLE rate $\lambda \sim \sqrt{\log d/n}$ and the constant was selected on Highway over $c \in \{0, 0.1, 0.2, 0.4\}$, with $c = 0.2$ winning at every budget. We transfer $c = 0.2$ unchanged to FourRooms and RecSim rather than retuning per environment. The per-task and pooled GenPQR baselines were swept over the same schedule ($c \in \{0, 0.05, 0.2, 0.5\}$) and no value improved them, so they are reported at $c = 0$; their off-support failure is structural rather than a regularization artifact. The IQ-Learn baselines do not use the classifier, so the schedule does not apply; their free parameter is the soft-Q temperature $\tau$, which we fix at $1.0$, the temperature at which the expert log-policy is defined.

\begin{table}[t]
\centering
\small
\begin{tabular}{lccc}
\toprule
 & FourRooms & Highway & RecSim \\
\midrule
Encoder                  & CNN        & MLP       & MLP        \\
Hidden width             & 512        & 128       & 256        \\
Residual blocks          & 0          & 0         & 2          \\
FQI iterations       & 16         & 16        & 4          \\
FQI regression epochs    & 50         & 50        & 40         \\
Basis rank $\rankbound$           & 4          & 4         & 4          \\
Nuclear-norm scalar $c$  & 0.2        & 0.2       & 0.2        \\
\midrule
MI clusters $K_{\mathrm{MI}}$  & 4    & 4         & 4          \\
MI E-step holdout        & 0.2        & 0.2       & 0.2        \\
\bottomrule
\end{tabular}
\caption{IRL pipeline settings per environment. All methods within an environment use
the same encoder and width, so the comparison isolates the estimation strategy.}
\label{tab:hparams}
\end{table}

\subsection{Metrics}
\label{app:metrics}

We report Pearson correlation between recovered and $\mu$-normalized ground-truth
rewards, which is invariant to the global scale and shift that IRL cannot identify. For each task, on-support states are those the task visits and off-support states are the remainder, where the state space is the union of states visited across all tasks. 

For transfer (Section~\ref{sec:transfer_xp}) we report return in the perturbed target as a fraction of the oracle ceiling, where the ceiling is the return of the optimal tabular policy for the true reward under the target dynamics. Perturbations lower the probability that a task's preferred topic appears in the candidate slate, swept over the degree of suppression; a random-uniform policy and a behavior-cloning policy bound the baselines.

\subsection{Computational efficiency}\label{sec:app_comp_efficiency}
Reward weights are sampled independently per agent. In Highway we draw $w \sim \mathcal{N}(0, I)$ over the four features (speed, collision, headway, lane change), constrain the collision weight to be non-positive since a positive one encodes seeking crashes, and normalize each row to unit norm. In FourRooms we draw $|\mathcal{N}(0, I)|$ over the four object colors and scale each row so its largest entry is one, confining the population to the positive orthant, RecSim samples mixing non-negative coefficients over a fixed rank-4 basis.

\begin{table}[t]
\centering
\small
\begin{tabular}{llrrr}
\toprule
Domain & Method & $N{=}8$ & $N{=}32$ & $N{=}128$ \\
\midrule
Highway, $n{=}500$ & \textbf{LowRank} & $+0.912 \pm 0.010$ & $+0.885 \pm 0.011$ & $+0.884 \pm 0.010$ \\
 & GenPQR-per-task & $+0.787 \pm 0.039$ & $+0.742 \pm 0.052$ & $+0.739 \pm 0.026$ \\
 & IQ-per-task & $+0.677 \pm 0.040$ & $+0.649 \pm 0.042$ & $+0.645 \pm 0.024$ \\
 & MI & $+0.849 \pm 0.014$ & $+0.846 \pm 0.082$ & $+0.878 \pm 0.003$ \\
\midrule
Highway, $n{=}100$ & \textbf{LowRank} & $+0.879 \pm 0.018$ & $+0.847 \pm 0.017$ & $+0.844 \pm 0.016$ \\
 & GenPQR-per-task & $-0.109 \pm 0.082$ & $-0.113 \pm 0.029$ & $-0.097 \pm 0.020$ \\
 & IQ-per-task & $-0.112 \pm 0.058$ & $-0.085 \pm 0.028$ & $-0.072 \pm 0.020$ \\
 & MI & $+0.304 \pm 0.204$ & $+0.645 \pm 0.057$ & $+0.815 \pm 0.028$ \\
\midrule
FourRooms, $n{=}500$ & \textbf{LowRank} & $+0.539 \pm 0.033$ & $+0.529 \pm 0.027$ & $+0.528 \pm 0.026$ \\
 & GenPQR-per-task & $+0.403 \pm 0.009$ & $+0.384 \pm 0.007$ & $+0.387 \pm 0.003$ \\
 & IQ-per-task & $+0.127 \pm 0.030$ & $+0.122 \pm 0.010$ & $+0.117 \pm 0.007$ \\
 & MI & $+0.323 \pm 0.032$ & $+0.259 \pm 0.030$ & $+0.243 \pm 0.021$ \\
\midrule
FourRooms, $n{=}100$ & \textbf{LowRank} & $+0.396 \pm 0.037$ & $+0.394 \pm 0.037$ & $+0.397 \pm 0.035$ \\
 & GenPQR-per-task & $+0.307 \pm 0.011$ & $+0.275 \pm 0.010$ & $+0.283 \pm 0.005$ \\
 & IQ-per-task & $+0.069 \pm 0.032$ & $+0.064 \pm 0.010$ & $+0.055 \pm 0.006$ \\
 & MI & $+0.241 \pm 0.076$ & $+0.198 \pm 0.021$ & $+0.198 \pm 0.010$ \\
\midrule
RecSim, $n{=}500$ & \textbf{LowRank} & $+0.776 \pm 0.015$ & $+0.775 \pm 0.018$ & $+0.776 \pm 0.015$ \\
 & GenPQR-per-task & $+0.196 \pm 0.010$ & $+0.190 \pm 0.004$ & $+0.183 \pm 0.002$ \\
 & IQ-per-task & $+0.134 \pm 0.010$ & $+0.124 \pm 0.006$ & $+0.117 \pm 0.002$ \\
 & MI & $+0.351 \pm 0.025$ & $+0.348 \pm 0.015$ & $+0.348 \pm 0.016$ \\
\midrule
RecSim, $n{=}100$ & \textbf{LowRank} & $+0.744 \pm 0.027$ & $+0.746 \pm 0.022$ & $+0.743 \pm 0.022$ \\
 & GenPQR-per-task & $+0.136 \pm 0.013$ & $+0.138 \pm 0.006$ & $+0.134 \pm 0.003$ \\
 & IQ-per-task & $+0.099 \pm 0.011$ & $+0.100 \pm 0.003$ & $+0.096 \pm 0.002$ \\
 & MI & $+0.336 \pm 0.028$ & $+0.337 \pm 0.021$ & $+0.343 \pm 0.015$ \\
\bottomrule
\end{tabular}
\caption{Reward-recovery quality (off-support global Pearson $r$, mean $\pm$ s.d.\ over 10 seeds) against population size $N$. LowRank and both per-task baselines are flat in $N$ (largest change is $0.048$); the small step from $N{=}8$ to $N{=}32$ on Highway and FourRooms is taken by all methods alike and reflects those populations sampling reward weights per agent, so agents $8$--$31$ differ from the first $8$. MI is the exception: it gains $+0.51$ on Highway at $n{=}100$, where its clusters are starved at $N{=}8$ (s.d. $0.204$, falling to $0.028$ by $N{=}128$), but loses $0.08$ on FourRooms at $n{=}500$ with tight error bars.}
\label{tab:scaling-quality}
\end{table}

\subsection{Coverage mechanism: derivation of the weighted-norm decomposition}
\label{app:coverage}

This appendix derives the eigenvalue decomposition of the weighted-norm penalty stated inline in Section~\ref{sec:coverage_xp}.

Let $\Theta^*$ be the true logits governing choices and $U^*$ its left singular vectors from the SVD $\Theta^* = U^*\Sigma V^{*\prime}$. For each state $s$, define the ground-truth design Gram matrix
\[
\Sigma_{col}(s) = \sum_i d_{i,s}\,(U^\star_{i,:})^\top U^\star_{i,:},
\]
with eigenpairs $\{\lambda_k(s), v_k(s)\}_{k=1}^{\rankbound}$ from its eigendecomposition, where $U^\star_{i,:}$ is the $i$-th row of $U^*$ and $w_i(s)$ is task $i$'s visit weight at state $s$.

Let $\Delta$ be the estimation error after the classification step, and let $U^*U^{*\prime}\Delta \equiv U^*A'$ be its projection onto the column space of $\Theta^*$, with coordinates $A' = U^{*\prime}\Delta$. Writing $a_{s,a}$ for the $(s,a)$ column of $A'$, the per-agent error is $\Delta_{i,(s,a)} = U^\star_{i,:} a_{s,a}$.
A state's contribution to the weighted norm is then
\begin{align*}
\sum_{i,a} d_i(s)\, \Delta_{i,(s,a)}^2
&= \sum_{i,a} d_i(s)\left(U^\star_{i,:} a_{s,a}\right)^2 \\
&= \sum_a a_{s,a}^{\prime}\left(\sum_i d_i(s)\, (U^\star_{i,:})^\top U^\star_{i,:}\right) a_{s,a} \\
&= \sum_a a_{s,a}^{\prime}\, \Sigma_{col}(s)\, a_{s,a} \\
&= \sum_j \lambda_j(s) \sum_a \langle a_{s,a}, v_j(s)\rangle^2,
\end{align*}
where the last equality applies the spectral theorem to $\Sigma_{col}(s)$. Low coverage along a direction is a small eigenvalue $\lambda_j(s) = \sum_i w_i(s)\left\langle (U_{i,:})^\top, v_j(s)\right\rangle^2$: when $\lambda_j(s)$ is small, a large error component $\langle a_{s,a}, v_j(s)\rangle$ contributes little to the penalty and hides in the weighted norm, though it can be large in the unweighted Frobenius norm that governs off-support recovery.

\subsection{Scaling Tasks}
Another possible strategy to scale across many tasks is a transfer learning approach: learning the basis, $\Phi$, for a relatively small pool of tasks (phase 1) and then for remainder or newly arriving tasks (phase 2), fitting the basis loadings $W$ via cross entropy loss for a frozen basis. The key is to curate the tasks for phase 1 carefully: choosing tasks that span the basis space for identification and have maximal data for statistical efficiency. 

\begin{figure}
\centering
\includegraphics[width=0.85\textwidth]{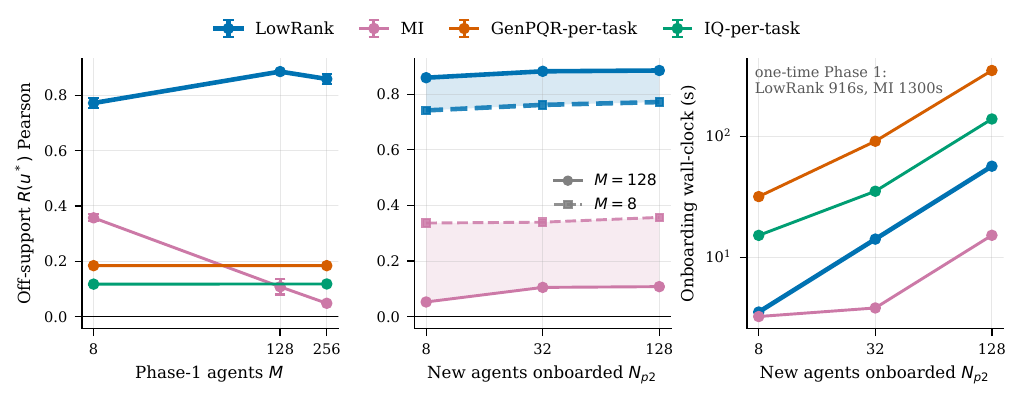}
\caption{Transfer learning is effective: training basis on $M$ agents in Phase 1 and transferring to $N_{p2}$ Phase 2 agents }
\label{fig:transfer_learning}
\end{figure}
Figure \ref{fig:transfer_learning} shows results from an experiment where we vary both the number of tasks in Phase 1, $M$, and the number of tasks in Phase 2, $N_{p2}$. Increasing $M$ by definition has no effect on per-task variants of GenPQR and IQ but has sizeable effects on LowRank and MI. For LowRank, increasing $M$ from 8 to 128 leads to a $\approx0.15$ increase in off-support Pearson correlation for new tasks in Phase 2 whereas the opposite is true of MI leading to a -0.25 decrease.

\subsection{Logit finite sample bound rate scaling}\label{app:rate-scaling}
To isolate each predicted scaling dependence in Theorem~\ref{thm:unweighted-regret}, we vary one dimension at a time holding fixed the others around the reference configuration: $N=8, S=81, A=5, \rankbound=2$, and $n=2^{22}$, at a regularization parameter chosen on a held-out dataset, $\lambda/\lambda_{\mathrm{theorem}}=1/32$. In the rank $\rankbound$ sweep, the truth is ($\Theta_{\rankbound}^\star=\rankbound^{-1/2}\sum_{j=1}^{\rankbound} U_{:,j}(V_{:,j})^\top$), with $\rankbound$ mutually orthogonal rank-one components, giving exact rank $\rankbound$ while preserving its Frobenius norm. The state-space sweep uses finer or coarser discretizations of the same fixed state domain, normalized to preserve signal strength. In the action-space sweep, we vary the number of slate positions and assign additional positions a balanced set of latent action profiles, rescaled to preserve total signal strength.

\begin{figure}
\centering
\includegraphics[width=0.85\textwidth]{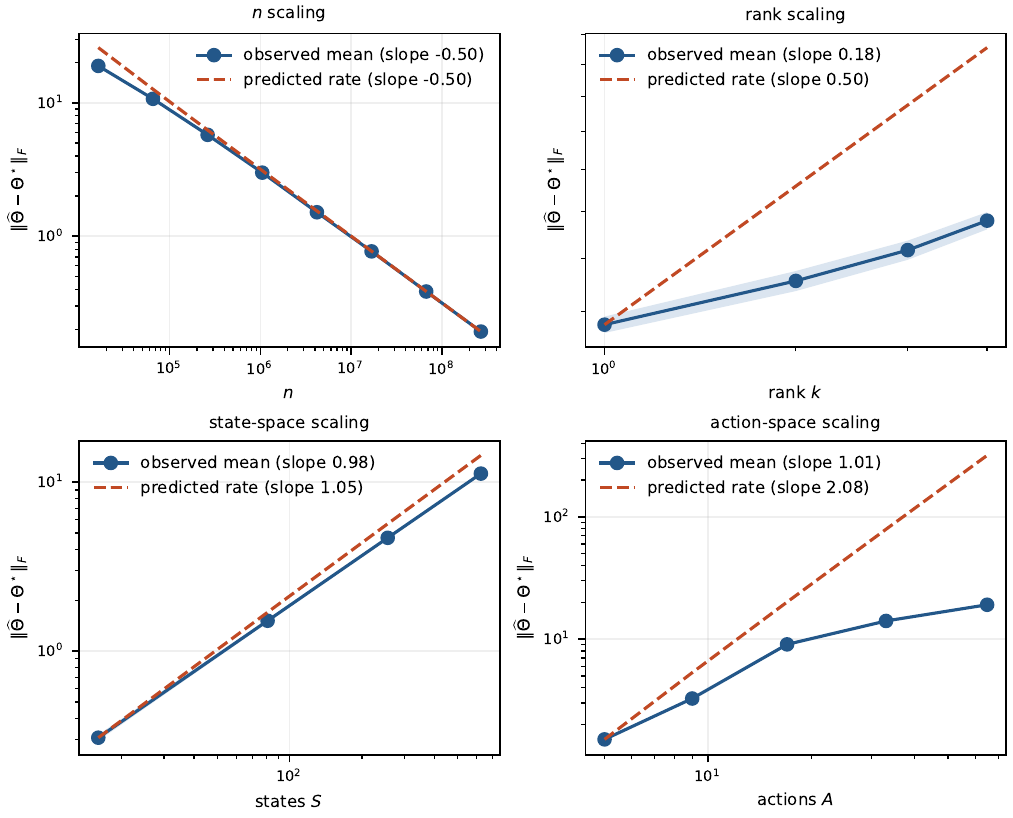}
\caption{Empirical scaling in RecSim matches the theory in $n$ and $S$ and is more conservative in $\rankbound$ and $A$.}
\label{fig:thm_scaling_xp}
\end{figure}

Figure~\ref{fig:thm_scaling_xp} shows the empirical scaling rate against that predicted by the theory. We match on sample-size and state-space cardinality while scaling more conservatively than the theory suggests for action-space cardinality and rank.

\subsection{Raw logit approximation}\label{app:raw_logit_xp}

In Section~\ref{sec:reward_recovery}, we make an approximation by applying the reward operator to the raw unnormalized logits, which represents large computational savings by allowing us to perform $\rankbound$ per-basis function versus $N$ per-task planning steps. To measure the effect of this approximation, we compare performance against the exact method, where planning is performed on the true log-policies per task.  

\begin{figure}
\centering
\includegraphics[width=\textwidth]{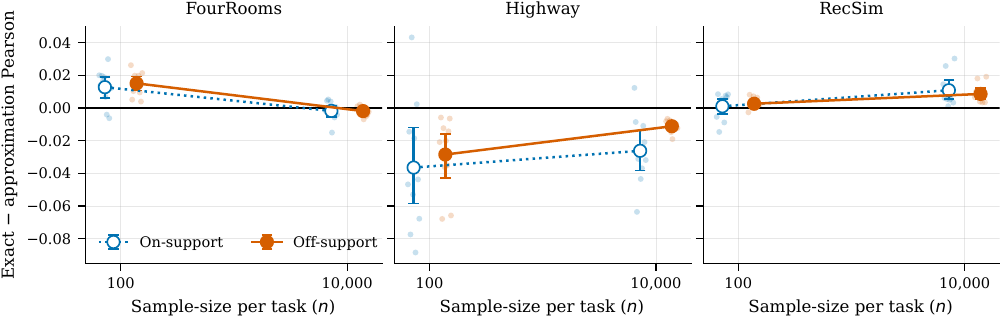}
\caption{Difference between reward recovery between reward operator applied to exact log normalized policies versus approximation}
\label{fig:exact_vs_approximation_ablation}
\end{figure}

Figure~\ref{fig:exact_vs_approximation_ablation} shows mild effects of using the approximation instead of the per-task normalized logits exactly. Using paired-bootstrap 95\% confidence intervals, for the 12 pairs of comparisons across data budget, environment and on vs off-support, 5 favor the exact method, 5 favor the approximation and 2 are indistinguishable from zero. The exact method is also slower: it takes approximately $1.1\times$ the wall-clock of the approximation for Highway ($\rankbound=4,N=5$) and $1.9\times$ and $1.8\times$ for FourRooms and RecSim ($\rankbound=4, N=8$). Wall-clock grows almost exactly with $N$, consistent with the exact path scaling per task ($N$ inversions) versus per basis column $\rankbound$ for the approximation.

\subsection{Rank Misspecification}\label{sec:rank}

A key assumption is that the log-policy matrix is approximately rank $\rankbound$. Because we learn $\rankbound$ basis functions and run FQI on each, we study robustness to rank misspecification: how recovery behaves when the true rank differs from $\rankbound$. This also probes the linear-reward assumption, since a reward that is nonlinear in the underlying features raises the true log-policy rank, one way the assumed $\rankbound$ can be too low. We expect graceful degradation, under-ranking underfits smoothly while over-ranking is benign, rather than a sharp failure.

\begin{figure}
\centering
\includegraphics[width=\textwidth]{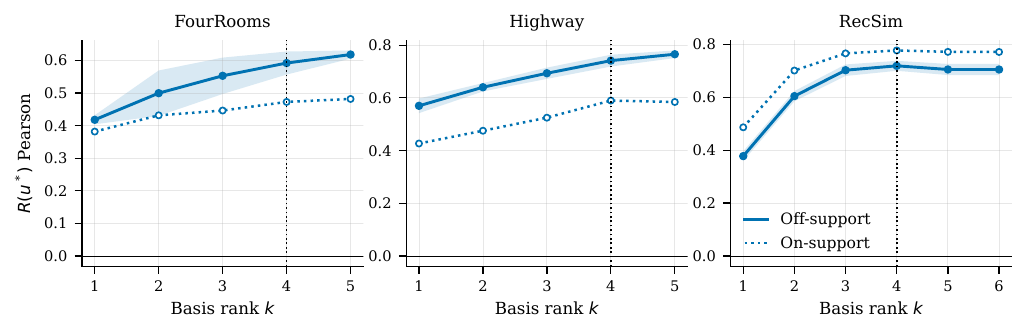}
\caption{Misspecification of rank leads to graceful degradation}
\label{fig:rank_ablation}
\end{figure}

Figure \ref{fig:rank_ablation} shows that as we both under and over-specify rank, recovery performance is smooth. Choosing $\rankbound$ lower than true rank leads to graceful degradation while $\rankbound$ higher leads to a plateau in performance. RecSim plateaus at $\rankbound=4$ matching the rank of the reward features whereas both FourRooms and Highway performance increases with $\rankbound$, related to the use of continuous states in those environments and hence greater rank of true log policies. This also demonstrates the value of applying low rank to log policies, instead of rewards, as we can check robustness to rank.

\subsection{Multi-intent scaling number of clusters}\label{sec:appendix_mi}
Table~\ref{tab:mi-ksweep} tests how scaling the number of clusters affects the multi-intent baseline, at a sample-size of $100$ and $10K$. The effect on scaling is mixed across environments and sample-sizes. Scaling clusters to the same number of tasks sometimes has the best performance (e.g FourRooms) but there are instances when an intermediate $K$ performs best (e.g Highway at $n=100$ and RecSim at $N=10^4$). Performance is lower than LowRank even for the best choice of $K_{MI}$ per cell.

\begin{table}[t]
\centering
\small
\setlength{\tabcolsep}{3pt}
\begin{tabular}{l *{6}{r@{\,$\pm$\,}l}}
\toprule
& \multicolumn{6}{c}{$n = 100$} & \multicolumn{6}{c}{$n = 10^4$} \\
\cmidrule(lr){2-7}\cmidrule(lr){8-13}
$K_{\mathrm{MI}}$
 & \multicolumn{2}{c}{Highway} & \multicolumn{2}{c}{FourRooms} & \multicolumn{2}{c}{RecSim}
 & \multicolumn{2}{c}{Highway} & \multicolumn{2}{c}{FourRooms} & \multicolumn{2}{c}{RecSim} \\
\midrule
$1$\rlap{\,\tiny$^{\dagger}$}
 & $0.04$ & $0.03$ & $0.01$ & $0.04$ & $0.01$ & $0.01$
 & $0.09$ & $0.01$ & $0.15$ & $0.00$ & $-0.00$ & $0.00$ \\
$2$
 & $\mathbf{0.23}$ & $\mathbf{0.13}$ & $0.13$ & $0.08$ & $0.04$ & $0.06$
 & $0.54$ & $0.03$ & $0.37$ & $0.06$ & $0.18$ & $0.01$ \\
$3$
 & $0.04$ & $0.07$ & $0.18$ & $0.07$ & $0.08$ & $0.06$
 & $0.60$ & $0.13$ & $0.41$ & $0.03$ & $0.30$ & $0.03$ \\
$4$
 & $-0.11$ & $0.06$ & $0.22$ & $0.05$ & $0.05$ & $0.04$
 & $0.68$ & $0.13$ & $0.43$ & $0.04$ & $\mathbf{0.34}$ & $\mathbf{0.02}$ \\
$6$
 & \multicolumn{2}{c}{---} & $0.27$ & $0.04$ & $0.10$ & $0.03$
 & \multicolumn{2}{c}{---} & $0.47$ & $0.03$ & $0.32$ & $0.01$ \\
$N$\rlap{\,\tiny$^{\ddagger}$}
 & $-0.17$ & $0.05$ & $\mathbf{0.34}$ & $\mathbf{0.02}$ & $\mathbf{0.13}$ & $\mathbf{0.02}$
 & $\mathbf{0.77}$ & $\mathbf{0.09}$ & $\mathbf{0.53}$ & $\mathbf{0.00}$ & $0.29$ & $0.01$ \\
\midrule
LowRank (ours)
 & $0.71$ & $0.02$ & $0.42$ & $0.08$ & $0.40$ & $0.02$
 & $0.84$ & $0.01$ & $0.72$ & $0.01$ & $0.72$ & $0.02$ \\
\bottomrule
\end{tabular}
\caption{MI-Cluster across cluster count $K_{\mathrm{MI}}$ (off-support Pearson $r$, mean $\pm$ s.d.\ over 10 seeds); $N{=}5$ for Highway, $8$ otherwise. Bold marks each column's best $K$. Ours beats MI at its best $K$ in every domain at both budgets. Shapes are budget-dependent: Highway rises to $K{=}N$ at $n{=}10^4$ but falls to it at $n{=}100$, and RecSim's peak at the generative rank $K{=}4$ disappears.}
\label{tab:mi-ksweep}
\end{table}

%% file: theory_appendix.tex
\section{Theory appendix}\label{sec:proofs}

\paragraph{Notations.}
For matrices, $\triplenorm{\cdot}_F$, $\triplenorm{\cdot}_*$, and $\triplenorm{\cdot}_{\op}$ denote the Frobenius, nuclear, and Euclidean operator norms, respectively. For matrices of the same size, $\traceinner{X}{Y}:=\tr(X^\top Y)$ denotes the Frobenius inner product. Vector norms use double bars; $\norm{x}_2$ is the Euclidean norm. We also write $\norm{\Theta}_\infty:=\max_{i,s,a}|\Theta_{i,(s,a)}|$ for the entrywise maximum norm. For $Z\in\R^{N\times SA}$, write $Z_{i,(s,\cdot)}:=(Z_{i,(s,a)})_{a\in\mathcal A}\in\R^A$ for its action block at task $i$ and state $s$.

\subsection{Setting}
\label{dn:sec:setup}

Let $\mathcal S$ and $\mathcal A$ be finite state and action spaces, with
$|\mathcal S|=S$ and $|\mathcal A|=A$. Assume there are $N$ tasks.
For each task $i\in[N]$, consider the MDPs:
\begin{equation*}
M_{\ell,i}:=(\mathcal S,\mathcal A,\mathsf P_\ell,r_i,d_0,\gamma_\ell),\qquad \ell\in\{1,2\}.
\end{equation*}
We call $M_{1,i}$ the source environment for task $i$ and call $M_{2,i}$ the target environment for task $i$.

We assume the MDPs share the initial state distribution $d_0$. For each $i$, the mean reward function $r_i$ is shared across $M_{1,i}$ and $M_{2,i}$. We assume the transition kernels $\mathsf P_1, \mathsf P_2$, the initial state distribution $d_0$, and the discount factors $0\leq\gamma_1,\gamma_2<1$ are known. 

We write $r_i(s,a)$ for the mean reward for task $i$ at state $s$ and action $a$. We assume each mean reward is finite.

Throughout the paper, we say a policy is soft optimal if it maximizes the Shannon entropy regularized value from every state with entropy coefficient $1$, as characterized by the soft Bellman equations in~\eqref{dn:eq:source-soft-Bellman}. Let $\pi_{1,i}^\star:\mathcal S\rightarrow \Delta(\mathcal{A})$ be the soft-optimal policy in the source environment $M_{1,i}$.

\paragraph{Observations from source environments $M_{1,i}$.} 

We observe independent tuples $(I_t,S_t,A_t)$, $t=1,\ldots,n$, generated by the following i.i.d. sampling procedure 
\begin{align}
    \Pp(I_t=i,S_t=s)&=d_{i,s},
    \label{dn:eq:source-design}\\
A_t\mid(I_t=i,S_t=s)&\sim\pi_{1,i}^\star(\cdot\mid s)\notag.
\end{align}
We assume the sampling probabilities $d_{i,s}$ for all $i\in [N]$ and $s\in \mathcal{S}$ are known and satisfy $d_{i,s}\geq 0$, $\sum_{i,s}d_{i,s}=1$. We write $d:= (d_{i,s})_{i\in [N], s\in \mathcal S}$.

We assume that each task $i\in [N]$ has positive marginal probability, where 
\begin{equation}
    p_i:=\sum_s d_{i,s}>0. \label{dn:eq:type-marginal}
\end{equation}

\paragraph{Candidate policies.}
Let $\Pi$ be the set of all stationary target policies. For each $\pi\in\Pi$ where $\pi:\mathcal S\to\Delta(\mathcal A)$, define
\begin{align}
    d_2^\pi(s) &:=(1-\gamma_2)\E_{\mathsf P_2,\pi,d_0} \left[\sum_{t\geq0}\gamma_2^t\one\{S_t=s\}\right], \label{dn:eq:target-state-occupancy}\\
    \rho_2^\pi(s,a)&:=d_2^\pi(s)\pi(a\mid s).\label{dn:eq:target-sa-occupancy}
\end{align}
We call $d_2^\pi$ the normalized discounted target state occupancy, and $\rho_2^\pi$ the normalized discounted target state-action occupancy. Note that both $d_2^\pi$ and $\rho_2^\pi$ are
known because the target dynamics and initial distribution are known.

For a candidate policy $\pi\in \Pi$, we write $\E_{\mathsf P_2,\pi,d_0}$ for the expectation  over a trajectory $(S_t,A_t)_{t\geq 0}$ generated by $S_0\sim d_0, A_t\sim \pi(\cdot\mid S_t)$, and $S_{t+1}\sim \mathsf P_2(\cdot\mid S_t, A_t)$. Then, for any bounded function $f:\mathcal S\times \mathcal A\to \R$, we have
\begin{align}
    \E_{\mathsf P_2,\pi,d_0}
    [\sum_{t\geq0}\gamma_2^t f(S_t,A_t)] &= \sum_s\left ( \sum_{t\geq 0}\gamma_2^t\E_{\mathsf P_2,\pi,d_0}[\one\{S_t=s\}] \right)\sum_a\pi(a\mid s)f(s,a) \notag\\
    &= \frac{1}{1-\gamma_2}\sum_s d_2^\pi(s)\sum_a \pi(a\mid s)f(s,a) \notag \\
    &=\frac1{1-\gamma_2}\sum_{s,a}\rho_2^\pi(s,a)f(s,a).
    \label{dn:eq:occupancy-identity}
\end{align}

\paragraph{Source logits.}
We fix a reference action $a_0\in\mathcal A$. Define 
$\Theta^\star\in\R^{N\times SA}$ by
\begin{equation}
    \Theta^\star_{i,(s,a)}
    :=\log\frac{\pi_{1,i}^\star(a\mid s)}{\pi_{1,i}^\star(a_0\mid s)},
    \qquad \Theta^\star_{i,(s,a_0)}=0.
    \label{dn:eq:anchored-logit}
\end{equation}
We call $\Theta^\star$ the true anchored source logit matrix. Note that the source soft-optimal policies have full support, so the log ratios $\Theta^\star_{i,(s,a)}$ are finite. 

Fix $\rankbound\leq \min \{N, SA\}$ and $B>0$.

\begin{assumption}
\label{dn:ass:low-rank}
Assume that the true source logit matrix $\Theta^\star$ satisfies $\Theta^\star\neq 0$ and
\begin{equation}
    \rank(\Theta^\star)\leq \rankbound,\qquad
    \norm{\Theta^\star}_\infty\leq B.
    \label{dn:eq:low-rank-bound}
\end{equation}
\end{assumption}
Later in Section~\ref{dn:app:zero-truth} we explain how the case of $\Theta^\star=0$, that is, if all source policies are uniform over actions independent of state, is resolved easily.

\paragraph{Reward identification.}
We use the reward normalization of \citet{van2025inverse} and the recovery identity in their Theorem 1. We impose the following reward normalization and fix a known policy $\mu:\mathcal S\to\Delta(\mathcal A)$ and a known state function $g:\mathcal S\to\R$. 

\begin{assumption}
\label{dn:ass:normalization}
For the fixed known $\mu$ and $g$ used in the procedure, every task's actual
shared mean reward satisfies
\begin{equation}
    \sum_{a\in\mathcal A}\mu(a\mid s)r_i(s,a)=g(s)
    \qquad\forall i\in [N],s\in\mathcal S.
    \label{dn:eq:reward-normalization}
\end{equation}
\end{assumption}

For a state function $V:\mathcal{S}\to\R$ and a state-action function $f:\mathcal{S}\times \mathcal{A} \to \R$, define linear operators
\begin{align*}
    (\mathsf P_1V)(s,a)&:=\sum_{s'}\mathsf P_1(s'\mid s,a)V(s'),& (\mu f)(s)&:=\sum_b\mu(b\mid s)f(s,b).
\end{align*}
For a state function $g:\mathcal S\to \R$, use the same symbol for its state-action extension $g(s,a):=g(s)$. 

The composition $\mathsf P_1\mu$ and the centering operator $C_\mu$ act on state-action functions as
\begin{align}
    (\mathsf P_1\mu f)(s,a) &:=\sum_{s',b}\mathsf P_1(s'\mid s,a)\mu(b\mid s')f(s',b), \label{dn:eq:P1-mu}\\
    (C_\mu f)(s,a)&:=f(s,a)-\sum_b\mu(b\mid s)f(s,b). \label{dn:eq:C-mu}
\end{align}

Write $u^\star\in\R^{N\times SA}$ where $u_i^\star(s,a):= \log \pi_{1,i}^\star(a\mid s)$. In matrix formulas, state-action functions such as $u_i^\star$, $r_i$, and $\rho_2^\pi$ are represented as column vectors in state then action order.
Under Assumption \ref{dn:ass:normalization}, we can recover the reward by 
\begin{equation}
r_i=C_\mu(I-\gamma_1\mathsf P_1\mu)^{-1}(u_i^\star-g)+g.
\label{dn:eq:normalized-reward-recovery}
\end{equation}
Here $I$ is the $SA\times SA$ identity. The inverse exists because $\mathsf P_1\mu$ is row-stochastic and $\gamma_1<1$.

\paragraph{Target objective.}
Candidate policies may assign zero probability to some actions and we use convention $0\log0=0$ throughout. For any target policy $\pi$ and task $i$, define 
\begin{equation}
    \softvalue_{2,i}(\pi;r_i)
    :=\E_{\mathsf P_2,\pi,d_0}\left[
    \sum_{t\geq0}\gamma_2^t
    \{r_i(S_t,A_t)-\log\pi(A_t\mid S_t)\}\right].
    \label{dn:eq:target-soft-value}
\end{equation}
We call $\softvalue_{2,i}$ the target soft value of policy $\pi$, which is the expectation of the total discounted reward plus Shannon entropy by running policy $\pi$ in $M_{2,i}$. By \eqref{dn:eq:occupancy-identity}, we have that for any fixed $\pi\in \Pi$, and reward functions $\tilde{r}_i, r_i: \mathcal{S}\times \mathcal A\to \R$,
\begin{equation}
\softvalue_{2,i}(\pi; \tilde{r}_i) - \softvalue_{2,i}(\pi; r_i) = \frac{1}{1-\gamma_2}\sum_{s,a} \rho_2^\pi(s,a)[\tilde{r}_i(s,a) - r_i(s,a)].
\end{equation}

Let $\pi_{2,i}^\star$ be the soft-optimal policy for $M_{2,i}$ from every state, characterized by its soft Bellman equations. In particular, $\pi_{2,i}^\star \in \arg\max_{\pi'}\softvalue_{2,i}(\pi'; r_i)$ and $\pi_{2,i}^\star\in\Pi$.

\subsection{Procedure: single logit fit}\label{dn:sec:unweighted-procedure}
Our procedure first fits the source logits, reconstructs the reward of each task, and selects the candidate policy with the largest estimated target soft value.

\paragraph{Logit estimation.}
Define the feasible set
\begin{equation}
    \mathcal M_B:=\left\{\Theta\in\R^{N\times SA}:
    \norm{\Theta}_\infty\leq B,\quad
    \Theta_{i,(s,a_0)}=0\ \text{for every }i,s\right\}.
    \label{dn:eq:parameter-space}
\end{equation}

For observation $t$, the negative log likelihood is given by $l_t(\Theta)=\log\left(\sum_{a\in\mathcal A} e^{\Theta_{I_t,(S_t,a)}}\right)-\Theta_{I_t,(S_t,A_t)}$, and the empirical loss is given by $L_n(\Theta):=\frac1n\sum_{t=1}^n l_t(\Theta)$.

Fix a failure probability $0<\delta<1$ and define
\begin{align}
\nu 
&:=\max\left\{
\left(1-\frac1A\right)\max_i p_i,\;
\frac12\max_s\sum_i d_{i,s}
\right\},
\label{dn:eq:unweighted-variance}\\*
\lambda_n
&:=2\left\{
\sqrt{\frac{2\nu }{n}\log\frac{2(N+SA)}{\delta}}
+\frac{2\sqrt2}{3n}\log\frac{2(N+SA)}{\delta}
\right\}.
\label{dn:eq:unweighted-lambda}
\end{align}
Note that this tuning uses total failure probability $\delta$ for a single logit matrix fit, so its error bound holds simultaneously across tasks $i$ and candidate $\pi\in \Pi$. 

We estimate the source logit matrix by
\begin{equation}
\widehat\Theta \in\argmin_{\Theta\in\mathcal M_B} \left\{L_n(\Theta)+\lambda_{n}
    \triplenorm{\Theta}_*\right\},
    \label{dn:eq:unweighted-estimator}
\end{equation}

Note that a minimizer exists because the objective is continuous on the nonempty compact set $\mathcal M_B$. 

\paragraph{Estimate the target value of each candidate policy}
Given $\widehat\Theta$, we define 
\begin{align}
    \widehat u_{i,(s,a)}
    &:=\widehat\Theta_{i,(s,a)}
      -\log\sum_b e^{\widehat\Theta_{i,(s,b)}},
      \label{dn:eq:unweighted-U-hat}\\
    \widehat r_i
    &:=C_\mu(I-\gamma_1\mathsf P_1\mu)^{-1}
       (\widehat u_i-g)+g.
      \label{dn:eq:unweighted-r-hat}
\end{align}
where $\widehat u_{i,(s,a)}$ is an estimate of the log probabilities $\log \pi_{1,i}^\star(a\mid s)$ and $\widehat r_i$ is an estimate of the reward of task $i$. 

Next, for each $\pi \in \Pi$, define the estimated target soft value by evaluating at reward $\widehat r_i$, that is,
\begin{equation}
    \widehat \softvalue_{2,i}(\pi):=\softvalue_{2,i}(\pi;\widehat r_i)=\frac{1}{1-\gamma_2}\sum_{s,a} \rho_2^\pi (s,a)[\widehat r_i(s,a) - \log \pi(a\mid s)].
\label{dn:eq:unweighted-J-hat}
\end{equation}

\paragraph{Plug-in policy selection.}
For each task $i$, choose
\begin{equation}
    \widehat\pi_i\in\argmax_{\pi\in\Pi} \widehat \softvalue_{2,i}(\pi).
    \label{dn:eq:unweighted-plugin-selector}
\end{equation}

\subsection{Main result: single logit fit}
\begin{theorem}
\label{dn:cor:unweighted-regret}
Under the setting of Section~\ref{dn:sec:setup} and
Assumptions~\ref{dn:ass:low-rank},~\ref{dn:ass:normalization} and~\ref{dn:ass:curvature}, use the procedure
of Section~\ref{dn:sec:unweighted-procedure}. With probability at least
$1-\delta$, the following statements hold simultaneously for every
minimizer in~\eqref{dn:eq:unweighted-estimator}, every task $i$, and every
candidate $\pi\in\Pi$ and every maximizing choice in~\eqref{dn:eq:unweighted-plugin-selector}.

\emph{(i) Estimation and value.} The common logit estimate satisfies
\begin{equation}
\triplenorm{\widehat\Theta-\Theta^\star}_F\leq\varepsilon_n,
\label{dn:eq:unweighted-logit-bound}
\end{equation}
and each estimated target value satisfies
\begin{equation}
\left|\widehat \softvalue_{2,i}(\pi)-\softvalue_{2,i}(\pi;r_i)\right|
\leq R_n(\pi).
\label{dn:eq:unweighted-value-certificate}
\end{equation}

\emph{(ii) Regret.} The regret of $\widehat{\pi}_i$ satisfies
\begin{equation}
\softvalue_{2,i}(\pi_{2,i}^\star;r_i)-\softvalue_{2,i}(\widehat\pi_i;r_i)
\leq R_n(\pi_{2,i}^\star)+R_n(\widehat\pi_i).
\label{dn:eq:unweighted-comparator-regret}
\end{equation}

\emph{(iii) Policy error.} 
\begin{equation}
\begin{aligned}
&\sum_s d_2^{\widehat\pi_i}(s)
\operatorname{KL}\!\left(\widehat\pi_i(\cdot\mid s)\,\middle\|\,\pi_{2,i}^\star(\cdot\mid s)\right)\\
&\leq(1-\gamma_2)\{R_n(\pi_{2,i}^\star)+R_n(\widehat\pi_i)\}\\
&=\varepsilon_n L_B\left\{\begin{aligned}
&\left\|\bigl[C_\mu(I-\gamma_1\Transition_1\mu)^{-1}\bigr]^\top\rho_2^{\pi_{2,i}^\star}\right\|_2\\
&+\left\|\bigl[C_\mu(I-\gamma_1\Transition_1\mu)^{-1}\bigr]^\top\rho_2^{\widehat\pi_i}\right\|_2
\end{aligned}\right\}.
\end{aligned}
\label{dn:eq:unweighted-realizable-KL}
\end{equation}
\end{theorem}

\begin{proof}[Proof sketch]
Part (i) follows from Lemmas~\ref{dn:lem:single-fit-estimation} and~\ref{dn:lem:general-value}; part (ii) from Lemma~\ref{dn:lem:general-selection}; and part (iii) from part (ii) and Lemma~\ref{dn:lem:soft-regret-KL}. All conclusions hold on the same estimation event.
\end{proof}

Note that the policy-error bounds \eqref{dn:eq:unweighted-realizable-KL} follow from \eqref{dn:eq:unweighted-comparator-regret} and the identity below
\begin{equation}
\softvalue_{2,i}(\pi_{2,i}^\star;r_i)-\softvalue_{2,i}(\pi;r_i)
=\frac{1}{1-\gamma_2}\sum_s d_2^\pi(s)
\operatorname{KL}\!\left(\pi(\cdot\mid s)\,\middle\|\,\pi_{2,i}^\star(\cdot\mid s)\right).
\label{dn:eq:unweighted-entropy-KL-identity}
\end{equation}

\subsection{Coverage condition}\label{dn:sec:unweighted-main}
\paragraph{Coverage condition.}
By Assumption~\ref{dn:ass:low-rank}, let $q:= \rank(\Theta^\star)\leq \rankbound$ and let $\Theta^\star=U\Sigma V^\top$ be a singular value decomposition of the true source logit matrix, with $U\in\R^{N\times q}$ and $V\in \R^{SA\times q}$ with orthonormal columns. We define the tangent space
\begin{align*}
T:=\{UX^\top+YV^\top:X\in\R^{SA\times q}, Y\in\R^{N\times q}\}.
\end{align*}
Let $P_{T}, P_{T^\perp}:\R^{N\times SA}\to \R^{N\times SA}$ denote orthogonal projection operators onto $T$ and its orthogonal complement $T^\perp$, that is, for every matrix $Z\in \R^{N\times SA}$, 
\begin{align*}
P_{T}Z :=UU^\top Z+ZVV^\top-UU^\top ZVV^\top, \qquad P_{T^\perp}Z=(I-UU^\top)Z(I-VV^\top).
\end{align*}
Note that $P_{T}, P_{T^\perp}$ are orthogonal projections with respect to the Frobenius inner product. 

Define 
\begin{equation}
\mathcal C (\Theta^\star):=\{\Delta:\Delta_{i,(s,a_0)}=0\ \forall i,s,\triplenorm{P_{T^\perp}\Delta}_*\leq3\triplenorm{P_{T}\Delta}_*\},
\label{dn:eq:unweighted-cone}
\end{equation}
and we call $\mathcal C (\Theta^\star)$ the anchored error cone. 

Define 
\begin{equation}
c_0(d,\Theta^\star) :=\inf_{\Delta\in\mathcal C (\Theta^\star)\setminus\{0\}} \frac{\sum_{i,s}d_{i,s}\norm{\Delta_{i,(s,\cdot)}}_2^2} {\triplenorm{\Delta}_F^2}.
\label{dn:eq:unweighted-restricted-coverage}
\end{equation}
We call $c_0$ the coverage coefficient. Note that the magnitude of $c_0$ is determined by the source design $d$ and the low-rank geometry of $\Theta^\star$. 

\begin{assumption}
\label{dn:ass:curvature}
$c_0(d,\Theta^\star)>0$.
\end{assumption}
Under Assumption \ref{dn:ass:curvature}, for every $\Delta\in\mathcal C(\Theta^\star)$, $\sum_{i,s}d_{i,s}\norm{\Delta_{i,(s,\cdot)}}_2^2 \geq c_0(d,\Theta^\star)\triplenorm{\Delta}_F^2$. 

\paragraph{Confidence radii for logit estimation.}
Let 
\begin{equation}
h_B:=\int_0^1(1-u)e^{-4Bu} du
=\frac{4B-1+e^{-4B}}{16B^2}>0.    \label{dn:eq:h-B}
\end{equation}
Define 
\begin{equation}
\varepsilon_n:=\frac{A^2 e^{2B}(260B+6/h_B)\sqrt \rankbound \lambda_n}{c_0(d,\Theta^\star) },\label{dn:eq:unweighted-logit-radius}
\end{equation}
The quantity $\varepsilon_n$ bounds the source-logit error in Frobenius norm in Theorem~\ref{dn:cor:unweighted-regret}. Substituting~\eqref{dn:eq:unweighted-lambda} gives
\begin{align*}
\varepsilon_n=\frac{A^2 e^{2B}(260B+6/h_B)}{c_0(d,\Theta^\star)}
\left\{2\sqrt{\frac{2\rankbound\nu \log(2(N+SA)/\delta)}n}
+\frac{4\sqrt{2\rankbound}\log(2(N+SA)/\delta)}{3n}\right\}.    
\end{align*}

\paragraph{Target value radius.}
Let 
\begin{equation}
L_B:=\sqrt{\frac{Ae^{2B}}{e^{2B}+A-1}},
\label{dn:eq:kappa-B}
\end{equation} and define 
\begin{equation}
R_n(\pi):=
\frac{\varepsilon_n}{1-\gamma_2}
L_B \left\| \left[C_\mu(I-\gamma_1\Transition_1\mu)^{-1}\right]^\top \rho_2^\pi \right\|_2.
\label{dn:eq:unweighted-value-radius}
\end{equation}
We call $R_n(\pi)$ a confidence radius for the target soft value of candidate policy $\pi$. 

Intuitively,  $\varepsilon_n$ measures uncertainty in the fitted logits, and the remaining factors in $R_n(\pi)$ describe how that error affects the target value of candidate policy $\pi$. 

\subsubsection{Interpreting the coverage condition}\label{dn:sec:coverage-interpretation}

For any $Z\in\R^{N\times SA}$, define 
\begin{equation}\label{dn:eq:observed-block-projection}
(P_d Z)_{i,(s,\cdot)}:=\one\{d_{i,s}>0\}Z_{i,(s,\cdot)}    
\end{equation}
which denotes the projection onto observed blocks, and let $z=\operatorname{vec}(Z)\in\R^{NSA}$ be a vector that stacks the rows of $Z$ in task, state, then action order. Define the symmetric matrices $\mathbf D_d,\mathbf P_T\in\R^{NSA\times NSA}$ by
\begin{align*}
\mathbf D_d:=\operatorname{diag}\bigl((d_{i,s})_{i,(s,a)}\bigr), \qquad \mathbf P_T\operatorname{vec}(Z):=\operatorname{vec}(P_TZ) \quad\forall Z\in\R^{N\times SA}.    
\end{align*}
where each $d_{i,s}$ is repeated over all $A$ actions and $\mathbf P_T$ is the matrix representation of the projection operator $P_T$.

\paragraph{Proofs of the coverage bounds}
\label{dn:app:coverage}

\paragraph{Restricted coverage and the observed blocks.}

\begin{lemma}
\label{dn:lem:general-coverage-nullspace}
Under Assumption~\ref{dn:ass:low-rank},
\begin{equation}
c_0>0
\quad\Longleftrightarrow\quad
\ker(P_d)\cap\mathcal C(\Theta^\star)=\{0\}.
\label{dn:eq:general-coverage-nullspace}
\end{equation}
where $P_d$ is the projection onto observed blocks defined in~\eqref{dn:eq:observed-block-projection}.
\end{lemma}

\begin{proof}
\begin{align*}
c_0
&=\inf_{\Delta\in\mathcal C(\Theta^\star)\setminus\{0\}}
\frac{\sum_{i,s}d_{i,s}\|\Delta_{i,(s,\cdot)}\|_2^2}
{\triplenorm{\Delta}_F^2}\\*
&=\inf_{\substack{\Delta\in\mathcal C(\Theta^\star)\\
\triplenorm{\Delta}_F=1}}
\sum_{i,s}d_{i,s}\|\Delta_{i,(s,\cdot)}\|_2^2\\*
&=\min_{\substack{\Delta\in\mathcal C(\Theta^\star)\\
\triplenorm{\Delta}_F=1}}
\sum_{i,s}d_{i,s}\|\Delta_{i,(s,\cdot)}\|_2^2.
\end{align*}
The first equality is~\eqref{dn:eq:unweighted-restricted-coverage}. For the second, divide each nonzero $\Delta$ by $\triplenorm{\Delta}_F$ which leaves the ratio unchanged and preserves the cone constraints~\eqref{dn:eq:unweighted-cone}. For the last equality, the unit-norm constraint and~\eqref{dn:eq:unweighted-cone} define a closed set. It is nonempty because $\Theta^\star\ne0$ belongs to the cone. Every matrix in this set satisfies
\[
\triplenorm{\Delta}_F^2
=\sum_{i,s}\|\Delta_{i,(s,\cdot)}\|_2^2
=1<\infty.
\]
Thus the feasible set is closed and bounded in finite dimension, hence compact; the continuous objective therefore attains its minimum.

If $0\ne\Delta\in\ker(P_d)\cap\mathcal C(\Theta^\star)$, then
\[
0\leq c_0
\leq\frac{\sum_{i,s}d_{i,s}\|\Delta_{i,(s,\cdot)}\|_2^2}
{\triplenorm{\Delta}_F^2}=0.
\]
The second inequality follows from~\eqref{dn:eq:unweighted-restricted-coverage}; the equality follows from $P_d\Delta=0$, which makes every block with $d_{i,s}>0$ vanish by~\eqref{dn:eq:observed-block-projection}. Hence $c_0>0$ implies $\ker(P_d)\cap\mathcal C(\Theta^\star)=\{0\}$.

Conversely, if $c_0=0$, let $\Delta_0$ attain the minimum above, so $\Delta_0\in\mathcal C(\Theta^\star)$ and $\triplenorm{\Delta_0}_F=1$. Then
\[
\sum_{i,s}d_{i,s}\|(\Delta_0)_{i,(s,\cdot)}\|_2^2=0.
\]
Every summand is nonnegative, so each block with $d_{i,s}>0$ must vanish:
\[
(\Delta_0)_{i,(s,\cdot)}=0
\qquad\text{whenever }d_{i,s}>0.
\]
By the definition~\eqref{dn:eq:observed-block-projection} of $P_d$, this gives
\[
P_d\Delta_0=0.
\]
Since $\Delta_0$ has unit norm, it is nonzero. Thus $c_0=0$ implies that the kernel and cone have a nonzero intersection, proving the reverse implication in~\eqref{dn:eq:general-coverage-nullspace}.
\end{proof}

\paragraph{A spectral lower bound.}

Use the singular value decomposition from Section~\ref{dn:sec:unweighted-main} and the matrices $\mathbf D_d,\mathbf P_T$ from Section~\ref{dn:sec:coverage-interpretation}.
The entries of $\mathbf P_T$ are
\begin{align*}
&(\mathbf P_T)_{(i,(s,a)),(i',(s',a'))}\\
={}&(UU^\top)_{i,i'}\one\{s=s',\ a=a'\} +\one\{i=i'\}(VV^\top)_{(s,a),(s',a')}\\
&\quad-(UU^\top)_{i,i'}(VV^\top)_{(s,a),(s',a')}.
\end{align*}

\begin{lemma}\label{dn:lem:general-coverage-spectrum}
Under Assumption~\ref{dn:ass:low-rank},
\begin{equation}
c_0\geq
\sup_{\tau\geq0}\left\{
\lambda_{\min}(\mathbf D_d-\tau\mathbf P_T)
+\frac{\tau}{1+18\rankbound}\right\}
\geq\min_{i,s}d_{i,s}.
\label{dn:eq:general-coverage-eigenvalue}
\end{equation}
\end{lemma}

\begin{proof}
Expanding the entries of $\mathbf P_T$ gives, for every $Z\in\R^{N\times SA}$,
\[
\mathbf P_T\operatorname{vec}(Z)
=\operatorname{vec}\!\left(UU^\top Z+ZVV^\top-UU^\top ZVV^\top\right)
=\operatorname{vec}(P_TZ).
\]
Since $P_T$ is an orthogonal projection, we have
\[
\mathbf P_T^\top=\mathbf P_T,\qquad
\mathbf P_T^2=\mathbf P_T,\qquad
\operatorname{vec}(Z)^\top\mathbf P_T\operatorname{vec}(Z)
=\triplenorm{P_TZ}_F^2.
\]

Take any $\Delta\in\mathcal C(\Theta^\star)\setminus\{0\}$ and set
$Z=\Delta$ and $z=\operatorname{vec}(Z)$. Then
\[
\begin{aligned}
\|z\|_2^2&=\triplenorm{Z}_F^2=\triplenorm{\Delta}_F^2>0,\\
z^\top\mathbf D_d z
=\sum_{i,s}d_{i,s}\sum_a Z_{i,(s,a)}^2 &=\sum_{i,s}d_{i,s}
\,\sum_a\Delta_{i,(s,a)}^2
=\sum_{i,s}d_{i,s}\|\Delta_{i,(s,\cdot)}\|_2^2.
\end{aligned}
\]
The projected matrix $P_TZ$ has rank at most $2\rankbound$ since
\[
\begin{aligned}
P_TZ&=UU^\top Z+(I-UU^\top)ZVV^\top,\\
\rank(P_TZ)
&\leq\rank(UU^\top Z)+\rank((I-UU^\top)ZVV^\top)
\leq2\rank(\Theta^\star)\leq2\rankbound.
\end{aligned}
\]
Therefore we have
\[
\begin{aligned}
\triplenorm{P_{T^\perp}Z}_F^2
&\leq\triplenorm{P_{T^\perp}Z}_*^2
\leq9\triplenorm{P_TZ}_*^2 \leq9\rank(P_TZ)\triplenorm{P_TZ}_F^2
\leq18\rankbound\triplenorm{P_TZ}_F^2.
\end{aligned}
\]
By orthogonality,
\begin{equation}
\begin{aligned}
\|z\|_2^2=\triplenorm{Z}_F^2
&=\triplenorm{P_TZ}_F^2+\triplenorm{P_{T^\perp}Z}_F^2\\
&\leq(1+18\rankbound)\triplenorm{P_TZ}_F^2
=(1+18\rankbound)z^\top\mathbf P_T z.
\end{aligned}
\label{dn:eq:weighted-tangent-energy}
\end{equation}
For every $\tau\geq0$, 
\[
\begin{aligned}
\frac{\sum_{i,s}d_{i,s}\|\Delta_{i,(s,\cdot)}\|_2^2}
{\triplenorm{\Delta}_F^2}
&=\frac{z^\top\mathbf D_d z}{\|z\|_2^2}\\
&=\frac{z^\top(\mathbf D_d-\tau\mathbf P_T)z}{\|z\|_2^2}
+\tau\frac{z^\top\mathbf P_T z}{\|z\|_2^2}\\
&\geq\lambda_{\min}(\mathbf D_d-\tau\mathbf P_T)
+\frac{\tau}{1+18\rankbound}.
\end{aligned}
\]
The last inequality uses $z^\top Mz/\|z\|_2^2\geq\lambda_{\min}(M)$ with the symmetric matrix $M=\mathbf D_d-\tau\mathbf P_T$. The second term is at least $\tau/(1+18\rankbound)$ by~\eqref{dn:eq:weighted-tangent-energy} and $\tau\geq0$. The right-hand side is independent of $\Delta$, so taking the infimum in~\eqref{dn:eq:unweighted-restricted-coverage} gives

\[
c_0\geq\lambda_{\min}(\mathbf D_d-\tau\mathbf P_T)
+\frac{\tau}{1+18\rankbound}
\qquad\text{for every }\tau\geq0.
\]
Taking the supremum of the right-hand side over $\tau\geq0$ proves the first inequality in~\eqref{dn:eq:general-coverage-eigenvalue}. Choosing $\tau=0$ gives 

\[
\sup_{\tau\geq0}
\left\{\lambda_{\min}(\mathbf D_d-\tau\mathbf P_T)
+\frac{\tau}{1+18\rankbound}\right\}
\geq\lambda_{\min}(\mathbf D_d)
=\min_{i,s}d_{i,s}.
\]
\end{proof}

\paragraph{Coverage characterization.}
The next lemma characterizes Assumption \ref{dn:ass:curvature} and gives a lower bound on $c_0$.

\begin{lemma}\label{dn:lem:unweighted-coverage}
Under Assumption~\ref{dn:ass:low-rank},
\begin{equation}
c_0(d,\Theta^\star) \geq\sup_{\tau\geq0} \left\{\lambda_{\min} \bigl(\mathbf D_d-\tau\mathbf P_T\bigr) +\frac{\tau}{1+18\rankbound}\right\} \geq\min_{i,s}d_{i,s}. \label{dn:eq:unweighted-coverage-eigenvalue-bound}
\end{equation}
where $\lambda_{\min}$ denotes the smallest eigenvalue of a matrix, so $\lambda_{\min}(\mathbf D_d-\tau\mathbf P_T) =\min_{\substack{z\in\R^{NSA}\\\|z\|_2=1}} z^\top(\mathbf D_d-\tau\mathbf P_T)z$.
Moreover, $c_0(d,\Theta^\star)>0$ if and only if $\ker(P_d)\cap\mathcal C(\Theta^\star)=\{0\}$.
\end{lemma}
\begin{proof}[Proof of Lemma~\ref{dn:lem:unweighted-coverage}]
The coverage characterization~\eqref{dn:eq:general-coverage-nullspace} and spectral lower bound~\eqref{dn:eq:general-coverage-eigenvalue} yield Lemma~\ref{dn:lem:unweighted-coverage}.
\end{proof}

Intuitively, the coverage coefficient $c_0(d,\Theta^\star)$ is strictly positive when every nonzero error in the cone changes at least one observed block. 

\begin{remark}
A positive lower bound in~\eqref{dn:eq:unweighted-coverage-eigenvalue-bound} gives a sufficient condition for identifying $\Theta^\star$ from the source population distribution against perturbations in $\mathcal C(\Theta^\star)$. Specifically, let $\Theta$ be an anchored source logit matrix
that induces the same source population distribution as $\Theta^\star$,
and suppose $\Delta:=\Theta-\Theta^\star\in\mathcal C(\Theta^\star)$.

First, on each block with $d_{i,s}>0$, equality of the population
distributions implies equality of the conditional action probabilities:
$\softmax(\Theta_{i,(s,\cdot)})=\pi_{1,i}^\star(\cdot\mid s)$.
Since $\Theta_{i,(s,a_0)}=0$, taking action-probability ratios gives
\begin{align*}
\Theta_{i,(s,a)}
=\log\frac{\pi_{1,i}^\star(a\mid s)}{\pi_{1,i}^\star(a_0\mid s)}
=\Theta^\star_{i,(s,a)}
\qquad\forall a,\quad d_{i,s}>0.    
\end{align*}
Thus $P_d\Delta=0$: the logits agree on every observed block.

Second, because $\Delta\in\mathcal C(\Theta^\star)$ and
$c_0>0$, definition~\eqref{dn:eq:unweighted-restricted-coverage}
gives
\begin{align*}
0=\sum_{i,s}d_{i,s}\norm{\Delta_{i,(s,\cdot)}}_2^2 \geq c_0\triplenorm{\Delta}_F^2.
\end{align*}

Hence $\Delta=0$ and $\Theta=\Theta^\star$, which establishes identification within the error cone. 
The single-fit estimation error belongs to this cone on the ``good event" which happens with high probability, as we later show in Step~2 of Section~\ref{dn:app:estimation}.
\end{remark}

\paragraph{Example with positive restricted coverage with an unobserved block.}
The following example shows how the lower bound certifies positive coverage coefficient for the single logit fit procedure even when some source task-state pairs are never observed.
\begin{example}
\label{dn:ex:missing-block}
Take $N=S=A=2$, $\rankbound=1$, $\ActionSpace=\{a_0,a_1\}$. Order matrix columns as $(1,a_0),(1,a_1),(2,a_0),(2,a_1)$ and set
\begin{align*}
(d_{i,s})_{i,s}=\begin{pmatrix}1/4&1/4\\1/2&0\end{pmatrix},
\qquad
\Theta^\star=\theta\begin{pmatrix}0&1&0&0\\0&0&0&0\end{pmatrix},
\qquad 0<\theta<B.    
\end{align*}
Then $p_1=p_2=1/2$. The task-$2$, state-$2$ block is unobserved.

For an anchored error matrix $\Delta$, let $Z=\Delta$. Its four potentially nonzero coordinates are $u,v,w,x$, one for action $a_1$ in each task-state block; $x$ corresponds to the unobserved block $(2,2)$. Write
\begin{align*}
Z=\begin{pmatrix}0&u&0&v\\0&w&0&x\end{pmatrix},
\qquad z=\operatorname{vec}(Z)=(0,u,0,v,0,w,0,x)^\top.    
\end{align*}
The tangent projection keeps the first row and the $(1,a_1)$ column:
\[
P_{T}Z=\begin{pmatrix}0&u&0&v\\0&w&0&0\end{pmatrix}.
\]
Consequently, both matrices in Lemma~\ref{dn:lem:unweighted-coverage} are diagonal:
\[
\mathbf D_d=\operatorname{diag}(1/4,1/4,1/4,1/4,1/2,1/2,0,0),
\qquad \mathbf P_{T}=\operatorname{diag}(1,1,1,1,0,1,0,0).
\]
The missing coordinate $x$ is perpendicular to the tangent space. For $\tau\geq0$,
\[
\begin{aligned}
\mathbf D_d-\tau\mathbf P_{T}
&=\operatorname{diag}(1/4-\tau,1/4-\tau,1/4-\tau,1/4-\tau,1/2,1/2-\tau,0,0),\\
\lambda_{\min}(\mathbf D_d-\tau\mathbf P_{T})
&=\min\{1/4-\tau,0\}.
\end{aligned}
\]
Since $\rankbound=1$, the expression in~\eqref{dn:eq:unweighted-coverage-eigenvalue-bound} is
\[
\lambda_{\min}(\mathbf D_d-\tau\mathbf P_{T})+\frac{\tau}{19}
=\begin{cases}
\tau/19,&0\leq\tau\leq1/4,\\[1mm]
1/4-18\tau/19,&\tau\geq1/4.
\end{cases}
\]
It increases up to $\tau=1/4$ and decreases thereafter. Lemma~\ref{dn:lem:unweighted-coverage} therefore gives
\[
c_0\geq\frac1{76}>0
\]

\end{example}

\subsection{Estimation for the single logit fit}
\label{dn:app:estimation}
We prove the statistical bound for the single-fit estimator in Section~\ref{dn:sec:unweighted-procedure}.

Write $D:=N+SA$.
\paragraph{Logit estimation.}
Use $\nu$, $\lambda_n$, and the estimator $\widehat\Theta$ from~\eqref{dn:eq:unweighted-variance}--\eqref{dn:eq:unweighted-estimator} in Section~\ref{dn:sec:unweighted-procedure}, with failure probability $0<\delta<1$.

\paragraph{Coverage condition.}
Use $q$, $T$, $P_T$, and $P_{T^\perp}$ from Section~\ref{dn:sec:unweighted-main}, the anchored error cone $\mathcal C(\Theta^\star)$ from~\eqref{dn:eq:unweighted-cone}, and the coverage coefficient $c_0(d,\Theta^\star)$ from~\eqref{dn:eq:unweighted-restricted-coverage}. We will use this anchored cone throughout Section~\ref{dn:app:estimation} and assume Assumption~\ref{dn:ass:curvature}. We abbreviate this coefficient as $c_0$ for fixed $d$ and $\Theta^\star$.

\paragraph{Curvature condition.}
To relate coverage coefficient $c_0$ to likelihood curvature, for each task state pair $i,s$, define matrix $H_{i,s}\in \R^{A\times A}$
\begin{equation}
H_{i,s}:=\diag\bigl(\pi_{1,i}^\star(\cdot\mid s)\bigr)
-\pi_{1,i}^\star(\cdot\mid s)\pi_{1,i}^\star(\cdot\mid s)^\top.
\label{dn:eq:Fisher-block}
\end{equation}
Define 
\begin{align}
Q(\Delta)&:=\sum_{i,s}d_{i,s}\Delta_{i,(s,\cdot)}^\top H_{i,s}\Delta_{i,(s,\cdot)},
\label{dn:eq:population-Fisher}\\*
Q_n(\Delta)&:=\frac1n\sum_{t=1}^n\Delta_{I_t,(S_t,\cdot)}^\top H_{I_t,S_t}\Delta_{I_t,(S_t,\cdot)}.
\label{dn:eq:empirical-Fisher}
\end{align}
We call $Q$ the Fisher population quadratic and $Q_n$ the Fisher empirical quadratic.

In Step~3 below we prove that bounded logits together with Assumption~\ref{dn:ass:curvature} imply
\begin{equation}\label{dn:eq:general-curvature}
Q(\Delta)\geq \frac{e^{-2B}c_0}{A^2} \triplenorm{\Delta}_F^2 \qquad\forall\Delta\in\mathcal C(\Theta^\star).
\end{equation}

\paragraph{Confidence radius for fitted logits.}
Use $h_B$ and $\varepsilon_n$ from~\eqref{dn:eq:h-B}--\eqref{dn:eq:unweighted-logit-radius}.

\paragraph{Logit estimation guarantee.}
\begin{lemma}
\label{dn:lem:single-fit-estimation}
Under the source sampling model in Section~\ref{dn:sec:setup} and
Assumption~\ref{dn:ass:low-rank}, suppose Assumption~\ref{dn:ass:curvature} holds. With probability at least $1-\delta$,
every minimizer in~\eqref{dn:eq:unweighted-estimator} satisfies
\begin{equation}
\begin{aligned}
&\triplenorm{\widehat\Theta-\Theta^\star}_F\\
&\leq\frac{A^2 e^{2B}(260B+6/h_B)}{c_0} \left\{2\sqrt{\frac{2\rankbound\nu\log(2 (N+SA)/\delta)}n}
+\frac{4\sqrt \rankbound\,\sqrt2\log(2 (N+SA)/\delta)}{3n}\right\}.
\end{aligned}
\label{dn:eq:general-error-bound}
\end{equation}
\end{lemma}

\begin{proof}[Proof of Lemma~\ref{dn:lem:single-fit-estimation}]
Use the tangent space $T$ defined in Section~\ref{dn:sec:unweighted-main}. For any error matrix $\Delta$, write
\[
\Delta_T:=P_T\Delta,
\qquad
\Delta_{T^\perp}:=P_{T^\perp}\Delta.
\]
\paragraph{Step 1: the good event and its probability.}
Let $e_i\in\R^N$ and $f_{(s,a)}\in\R^{SA}$ denote the canonical basis vectors and define
\begin{equation}
X(i,s,a):=e_if_{(s,a)}^\top.
\label{dn:eq:basis-matrix}
\end{equation}
The gradient of the empirical loss at the truth is
\begin{equation}
\nabla L_n(\Theta^\star)
=
\frac1n\sum_{t=1}^n
\sum_{a\in\ActionSpace}
\left\{
\pi_{1,I_t}^\star(a\mid S_t)-\one\{A_t=a\}
\right\}
X(I_t,S_t,a).
\label{dn:eq:score-summand}
\end{equation}
Conditional on $(I_t,S_t)=(i,s)$,
\[
\E[e_{A_t}\mid i,s]=\pi_{1,i}^\star(\cdot\mid s).
\]
Define the score summand
\begin{equation}
Y_t
:=
\sum_{a\in\ActionSpace}
\left\{
\pi_{1,I_t}^\star(a\mid S_t)-\one\{A_t=a\}
\right\}
X(I_t,S_t,a).
\label{dn:eq:weighted-score}
\end{equation}
Thus $\E[Y_t]=0$, and
\begin{equation}
\triplenorm{
\nabla L_n(\Theta^\star)
}_{\op}
=
\triplenorm{\frac1n\sum_{t=1}^nY_t}_{\op}.
\label{dn:eq:weighted-score-average}
\end{equation}

For any probability vector $q\in\R^A$ and a categorical random action $A\sim q$,
\begin{equation}
\E[(q-e_A)(q-e_A)^\top]
=\diag(q)-qq^\top,
\qquad
\E\|q-e_A\|_2^2
=1-\|q\|_2^2.
\label{dn:eq:categorical-identities}
\end{equation}
Because $Y_t$ has one nonzero row and one nonzero state block, direct multiplication gives
\begin{align*}
\triplenorm{\E[Y_tY_t^\top]}_{\op}
&=
\max_i\sum_s
d_{i,s}
\left(1-\|\pi_{1,i}^\star(\cdot\mid s)\|_2^2\right),\\
\triplenorm{\E[Y_t^\top Y_t]}_{\op}
&=
\max_s
\triplenorm{
\sum_i
d_{i,s}H_{i,s}
}_{\op}.
\end{align*}
For the first variance term, Cauchy--Schwarz gives, for every probability vector $q$,
\[
1=\left(\sum_aq_a\right)^2
\leq A\sum_aq_a^2,
\qquad
1-\|q\|_2^2\leq1-\frac1A.
\]
Applying this to $q=\pi_{1,i}^\star(\cdot\mid s)$ in each summand yields
\[
\triplenorm{\E[Y_tY_t^\top]}_{\op}
\leq
\left(1-\frac1A\right)
\max_i\sum_sd_{i,s}.
\]
For the second variance term, the symmetric matrix $\diag(q)-qq^\top$ has operator norm at most its largest absolute row sum.  Its diagonal entries are $q_a(1-q_a)$ and its off-diagonal entries are $-q_aq_b$, so
\begin{equation}
\triplenorm{\diag(q)-qq^\top}_{\op}
\leq
\max_a\left\{q_a(1-q_a)+\sum_{b\ne a}q_aq_b\right\}
=\max_a2q_a(1-q_a)
\leq\frac12.
\label{dn:eq:Fisher-operator-bound}
\end{equation}
Thus $\triplenorm{H_{i,s}}_{\op}\leq1/2$.  The triangle inequality and the nonnegative weights then give
\[
\triplenorm{\E[Y_t^\top Y_t]}_{\op}
\leq
\max_s\sum_i
d_{i,s}
\triplenorm{H_{i,s}}_{\op}
\leq
\frac12\max_s\sum_id_{i,s}.
\]
Consequently, both matrix variance terms are bounded by
\[
\nu
=
\max\left\{
\left(1-\frac1A\right)
\max_i\sum_sd_{i,s},
\quad
\frac12\max_s\sum_id_{i,s}
\right\}.
\]
Furthermore, for every probability vector $q$ and action $a$, $\|q-e_a\|_2^2=1-2q_a+\|q\|_2^2\leq2$.  Hence
\begin{align}
\triplenorm{Y_t}_{\op}
&=
\|\pi_{1,I_t}^\star(\cdot\mid S_t)-e_{A_t}\|_2
\notag \leq
\sqrt2.
\label{dn:eq:score-envelope}
\end{align}
Rectangular matrix Bernstein~\cite[Theorem~6.1.1]{Tropp2015} therefore gives, for every $x>0$,
\begin{equation}
\Pp\left[
\triplenorm{\frac1n\sum_{t=1}^nY_t}_{\op}
>
\sqrt{\frac{2\nu x}{n}}
+
\frac{2\sqrt2 x}{3n}
\right]
\leq De^{-x}.
\label{dn:eq:matrix-Bernstein-score}
\end{equation}
Following the terminology of Wainwright~\cite[Section~10.2.2, proof of Proposition~10.6]{wainwright2019high}, the \emph{good event} is the event that the likelihood gradient at the truth satisfies
\begin{align}
\triplenorm{
\nabla L_n(\Theta^\star)
}_{\op}
&\leq\frac{\lambda_n}2  = \sqrt{\frac{2\nu}{n}\log\frac{2(N+SA)}{\delta}} + \frac{2\sqrt2}{3n}\log\frac{2(N+SA)}{\delta}.
\label{dn:eq:good-event}
\end{align}
Blocks of zero source probability contribute zero to the variance sums; no division by $d_{i,s}$ was used. Taking $x=\log\frac{2(N+SA)}{\delta}$ in~\eqref{dn:eq:matrix-Bernstein-score} shows that the good event~\eqref{dn:eq:good-event} has probability at least $1-\delta/2$.

\paragraph{Step 2: cone condition and upper bound on $D_{L_n}$.}
Work on the good event~\eqref{dn:eq:good-event}. Set
\[
\Delta:=\widehat\Theta-\Theta^\star.
\]
Define the Bregman divergence
\begin{equation}
D_{L_n}(\Theta^\star+\Delta,\Theta^\star)
:=
L_n(\Theta^\star+\Delta)-L_n(\Theta^\star)
-
\traceinner{\nabla L_n(\Theta^\star)}{\Delta}.
\label{dn:eq:Bregman}
\end{equation}
Optimality of $\widehat\Theta$ gives
\begin{align}
D_{L_n}(\Theta^\star+\Delta,\Theta^\star) &\leq -\traceinner{\nabla L_n(\Theta^\star)}{\Delta} + \lambda_n \left\{ \triplenorm{\Theta^\star}_* - \triplenorm{\Theta^\star+\Delta}_* \right\}.
\label{dn:eq:basic-inequality}
\end{align}
The good event~\eqref{dn:eq:good-event} and nuclear/operator duality imply
\begin{equation}
\left|
\traceinner{\nabla L_n(\Theta^\star)}{\Delta}
\right|
\leq
\frac{\lambda_n}2
\triplenorm{\Delta}_*.
\label{dn:eq:score-duality}
\end{equation}
By the definition of $T^\perp$, the row and column spaces of $\Delta_{T^\perp}$ are orthogonal to those of $\Theta^\star$.  In orthonormal bases adapted to these spaces, their sum is block diagonal, so the nuclear norms add:
\[
\triplenorm{\Theta^\star+\Delta_{T^\perp}}_*
=
\triplenorm{\Theta^\star}_*
+
\triplenorm{\Delta_{T^\perp}}_*.
\]
Using $\Delta=\Delta_T+\Delta_{T^\perp}$ and the reverse triangle inequality,
\begin{align*}
\triplenorm{\Theta^\star+\Delta}_*
&=
\triplenorm{\Theta^\star+\Delta_{T^\perp}+\Delta_T}_*\\*
&\geq
\triplenorm{\Theta^\star+\Delta_{T^\perp}}_*
-\triplenorm{\Delta_T}_*\\*
&=
\triplenorm{\Theta^\star}_*
+\triplenorm{\Delta_{T^\perp}}_*
-\triplenorm{\Delta_T}_*.
\end{align*}
Rearranging gives
\begin{equation}
\triplenorm{\Theta^\star}_*
-
\triplenorm{\Theta^\star+\Delta}_*
\leq
\triplenorm{\Delta_T}_*
-
\triplenorm{\Delta_{T^\perp}}_*.
\label{dn:eq:decomposability}
\end{equation}
Combining~\eqref{dn:eq:basic-inequality}--\eqref{dn:eq:decomposability},
\begin{equation}
D_{L_n}(\Theta^\star+\Delta,\Theta^\star)
\leq
\frac{3\lambda_n}2
\triplenorm{\Delta_T}_*
-
\frac{\lambda_n}2
\triplenorm{\Delta_{T^\perp}}_*.
\label{dn:eq:cone-inequality}
\end{equation}
The upper bound~\eqref{dn:eq:cone-inequality} uses the good event~\eqref{dn:eq:good-event} through the gradient estimate~\eqref{dn:eq:score-duality}.  Since the Bregman divergence is nonnegative by convexity of $L_n$ and $\lambda_n>0$, we obtain, on this good event,
\begin{equation}
\triplenorm{\Delta_{T^\perp}}_*
\leq
3\triplenorm{\Delta_T}_*.
\label{dn:eq:cone-condition}
\end{equation}
A tangent-space matrix has rank at most $2\rankbound$, and therefore
\begin{align}
\triplenorm{\Delta}_*
&\leq4\triplenorm{\Delta_T}_*
\leq4\sqrt{2\rankbound}\triplenorm{\Delta}_F,
\label{dn:eq:cone-size}\\
D_{L_n}(\Theta^\star+\Delta,\Theta^\star)
&\leq
\frac{3\lambda_n}2
\sqrt{2\rankbound}\triplenorm{\Delta}_F.
\label{dn:eq:Bregman-upper}
\end{align}
On the good event~\eqref{dn:eq:good-event}, $\Delta\in\mathcal C(\Theta^\star)$, so
\[
\mathcal C(\Theta^\star)=\{0\}\quad\Longrightarrow\quad\Delta=0.
\]
The remaining steps concern $\mathcal C(\Theta^\star)\ne\{0\}$.

\paragraph{Step 3: likelihood and population restricted strong convexity.}
Recall $Q,Q_n$ from~\eqref{dn:eq:population-Fisher}--\eqref{dn:eq:empirical-Fisher}.

\begin{claim}
Every anchored matrix satisfies
\begin{equation}
Q(\Delta)\geq\frac{e^{-2B}}{A^2}\sum_{i,s}d_{i,s}\|\Delta_{i,(s,\cdot)}\|_2^2.
\label{dn:eq:population-curvature}
\end{equation}
\end{claim} 

\begin{proof}[Proof of the claim]
We first bound each block below, and then sum the inequalities under source weights.

Fix $(i,s)$.  The bounded logits give, for every action $a$,
\[
\pi_{1,i}^\star(a\mid s)
=\frac{e^{\Theta^\star_{i,(s,a)}}}{\sum_b e^{\Theta^\star_{i,(s,b)}}}
\geq\frac{e^{-B}}{Ae^B}
=\frac{e^{-2B}}A.
\]
The quadratic form $\Delta_{i,(s,\cdot)}^\top H_{i,s}\Delta_{i,(s,\cdot)}$ equals the variance of $\Delta_{i,(s,a)}$ under $a\sim\pi_{1,i}^\star(\cdot\mid s)$. Writing this variance as the minimum squared deviation from a constant gives
\begin{align*}
\Delta_{i,(s,\cdot)}^\top H_{i,s}\Delta_{i,(s,\cdot)}
&=\min_{c\in\R}\sum_a\pi_{1,i}^\star(a\mid s)(\Delta_{i,(s,a)}-c)^2\\*
&\geq\frac{e^{-2B}}A\min_{c\in\R}\sum_a(\Delta_{i,(s,a)}-c)^2\\*
&=\frac{e^{-2B}}A\left[\|\Delta_{i,(s,\cdot)}\|_2^2-\frac1A\left(\sum_a\Delta_{i,(s,a)}\right)^2\right].
\end{align*}
Anchoring controls the subtracted term: since $\Delta_{i,(s,a_0)}=0$, Cauchy--Schwarz gives
\[
\left(\sum_a\Delta_{i,(s,a)}\right)^2
=\left(\sum_{a\ne a_0}\Delta_{i,(s,a)}\right)^2
\leq(A-1)\sum_{a\ne a_0}\Delta_{i,(s,a)}^2
=(A-1)\|\Delta_{i,(s,\cdot)}\|_2^2.
\]
Substitution therefore yields the required bound for each block:
\begin{equation}
\Delta_{i,(s,\cdot)}^\top H_{i,s}\Delta_{i,(s,\cdot)}
\geq\frac{e^{-2B}}A\left(1-\frac{A-1}A\right)\|\Delta_{i,(s,\cdot)}\|_2^2
=\frac{e^{-2B}}{A^2}\|\Delta_{i,(s,\cdot)}\|_2^2.
\label{dn:eq:anchored-curvature}
\end{equation}
Finally, multiply by $d_{i,s}\geq0$ and sum to obtain~\eqref{dn:eq:population-curvature}. For a cone direction, definition~\eqref{dn:eq:unweighted-restricted-coverage} supplies the stated comparison with the Frobenius norm.
\end{proof}

When $\mathcal C(\Theta^\star)\ne\{0\}$, every $\Delta\in\mathcal C(\Theta^\star)$ therefore satisfies
\[
Q(\Delta)
\geq\frac{e^{-2B}c_0}{A^2}\triplenorm{\Delta}_F^2,
\]
by the definition~\eqref{dn:eq:unweighted-restricted-coverage}. This proves~\eqref{dn:eq:general-curvature}.

We next prove the Hessian comparison~\eqref{dn:eq:Hessian-path} for the log-sum-exp function along any line segment in $[-B,B]^A$. Let $z,z'\in[-B,B]^A$, $q_u=\softmax(z+u(z'-z))$, and $q_0=\softmax(z)$.  Since every coordinate of $z'-z$ lies in $[-2B,2B]$,
\begin{equation}
\frac{q_u(a)}{q_0(a)}
=
\frac{e^{u(z'_a-z_a)}}
{\sum_bq_0(b)e^{u(z'_b-z_b)}}
\geq e^{-4Bu}.
\label{dn:eq:softmax-path-ratio}
\end{equation}
Using
\[
v^\top[\diag(q)-qq^\top]v
=
\min_c\sum_aq_a(v_a-c)^2,
\]
coordinatewise domination in~\eqref{dn:eq:softmax-path-ratio} implies
\begin{equation}
\diag(q_u)-q_uq_u^\top
\succeq
 e^{-4Bu}\{\diag(q_0)-q_0q_0^\top\}.
\label{dn:eq:Hessian-path}
\end{equation}
To apply this comparison to the empirical loss, fix an observation $t$ and take $z=\Theta^\star_{I_t,(S_t,\cdot)}$ and $z'=\Theta^\star_{I_t,(S_t,\cdot)}+\Delta_{I_t,(S_t,\cdot)}$.  Both blocks lie in $[-B,B]^A$ because the truth and the estimate are feasible.  The linear term in $l_t$ has zero second derivative, while the Hessian of log-sum-exp is the categorical covariance.  Thus, with $q_u$ corresponding to this block,
\[
\frac{d^2}{du^2}l_t(\Theta^\star+u\Delta)
=\Delta_{I_t,(S_t,\cdot)}^\top[\diag(q_u)-q_uq_u^\top]\Delta_{I_t,(S_t,\cdot)}
\geq e^{-4Bu}\Delta_{I_t,(S_t,\cdot)}^\top H_{I_t,S_t}\Delta_{I_t,(S_t,\cdot)}.
\]
Averaging over $t$ gives
\[
\frac{d^2}{du^2}L_n(\Theta^\star+u\Delta)
\geq\frac{e^{-4Bu}}n\sum_{t=1}^n
\Delta_{I_t,(S_t,\cdot)}^\top H_{I_t,S_t}\Delta_{I_t,(S_t,\cdot)}
=e^{-4Bu}Q_n(\Delta).
\]
The Bregman divergence is exactly the integral remainder after subtracting the first-order term.  Taylor's theorem therefore gives
\begin{align}
D_{L_n}(\Theta^\star+\Delta,\Theta^\star)
&=L_n(\Theta^\star+\Delta)-L_n(\Theta^\star)
-\traceinner{\nabla L_n(\Theta^\star)}{\Delta}
\notag\\*
&=\int_0^1(1-u)\frac{d^2}{du^2}L_n(\Theta^\star+u\Delta)\,du
\notag\\*
&\geq
\left[\int_0^1(1-u)e^{-4Bu}\,du\right]Q_n(\Delta)
\notag\\*
&=
\frac{4B-1+e^{-4B}}{16B^2}Q_n(\Delta).
\label{dn:eq:Bregman-lower}
\end{align}
The inequality holds by the preceding second-derivative bound and $1-u\geq0$; the last equality is $h_B$ from~\eqref{dn:eq:h-B}. 
For each fixed feasible $\Delta$, taking expectations gives $D_L(\Theta^\star+\Delta,\Theta^\star)\geq h_BQ(\Delta)$ for the population loss $L:=\E L_n$. Together with~\eqref{dn:eq:general-curvature}, this yields
\[
D_L(\Theta^\star+\Delta,\Theta^\star)
\geq\frac{h_Be^{-2B}c_0}{A^2}\triplenorm{\Delta}_F^2,
\qquad \Delta\in\mathcal C(\Theta^\star),\quad\Theta^\star+\Delta\in\mathcal M_B,
\]
which establishes population restricted strong convexity (RSC) on feasible cone directions; see Wainwright~\cite[Chapter~9, Section~9.3.1]{wainwright2019high}.

\paragraph{Step 4: empirical-process control at one deterministic radius.}
We bound the empirical-process error uniformly over eligible errors of Frobenius norm at most a fixed deterministic $t$, with failure probability at most $\delta/2$. Fix a deterministic $t>0$. Define
\begin{equation}
\mathcal B(t):=\{\Delta\in\mathcal C(\Theta^\star):
\triplenorm{\Delta}_F\leq t,\quad\norm{\Delta}_\infty\leq2B\}.
\label{dn:eq:fixed-radius-class}
\end{equation}

Write $f_\Delta(i,s)=\Delta_{i,(s,\cdot)}^\top H_{i,s}\Delta_{i,(s,\cdot)}$. Let $q_a:=\pi_{1,i}^\star(a\mid s)$. By the definition
$H_{i,s}=\diag(q)-qq^\top$,
\[
\begin{aligned}
f_\Delta(i,s)
&=\Delta_{i,(s,\cdot)}^\top[\diag(q)-qq^\top]\Delta_{i,(s,\cdot)}\\
&=\sum_a q_a\Delta_{i,(s,a)}^2
-\left(\sum_a q_a\Delta_{i,(s,a)}\right)^2\\
&=\E_{a\sim q}[\Delta_{i,(s,a)}^2]-\bigl(\E_{a\sim q}[\Delta_{i,(s,a)}]\bigr)^2=\Var_{a\sim q}(\Delta_{i,(s,a)}).
\end{aligned}
\]
Here $(i,s)$ and $\Delta$ are fixed. Thus $f_\Delta(i,s)\geq0$, and dropping the
nonnegative squared mean gives
$f_\Delta(i,s)\leq\sum_a q_a\Delta_{i,(s,a)}^2$.
For $\Delta\in\mathcal B(t)$, the constraint
$\norm{\Delta}_\infty\leq2B$ implies
$\Delta_{i,(s,a)}^2\leq4B^2$ for every action. Since $\sum_a q_a=1$,
\[
\sum_a q_a\Delta_{i,(s,a)}^2\leq4B^2\sum_a q_a=4B^2.
\]
Combining these three observations yields
\begin{equation}
0\leq f_\Delta(i,s)\leq\sum_a\pi_{1,i}^\star(a\mid s)\Delta_{i,(s,a)}^2\leq4B^2.
\label{dn:eq:fixed-radius-envelope}
\end{equation}
\emph{Bounding the Fisher operator norm.}
We next justify $\triplenorm{H_{i,s}}_{\op}\leq1/2$ directly.
For a fixed block, retain $q_a=\pi_{1,i}^\star(a\mid s)$. The entries of
$H_{i,s}=\diag(q)-qq^\top$ are
\[
(H_{i,s})_{aa}=q_a(1-q_a),\qquad
(H_{i,s})_{ab}=-q_aq_b\quad(a\ne b).
\]
Thus the sum of absolute values in row $a$ satisfies
\[
\begin{aligned}
\sum_b|(H_{i,s})_{ab}|
&=q_a(1-q_a)+\sum_{b\ne a}q_aq_b\\
&=q_a(1-q_a)+q_a(1-q_a)\\
&=2q_a(1-q_a)\leq\frac12.
\end{aligned}
\]
The second equality uses $\sum_{b\ne a}q_b=1-q_a$; the last inequality
uses $q_a(1-q_a)=1/4-(q_a-1/2)^2\leq1/4$.

\begin{samepage}
For any real matrix $M$ and any vector $x$ of the appropriate dimension,
the triangle inequality and Cauchy--Schwarz give, for each row $a$,
\[
\begin{aligned}
|(Mx)_a|^2
&=\left|\sum_b M_{ab}x_b\right|^2\\
&\leq\left(\sum_b|M_{ab}|\,|x_b|\right)^2\\
&=\left(\sum_b\sqrt{|M_{ab}|}\,
\sqrt{|M_{ab}|}\,|x_b|\right)^2\\
&\leq
\left(\sum_b|M_{ab}|\right)
\left(\sum_b|M_{ab}|\,|x_b|^2\right).
\end{aligned}
\]
\end{samepage}
Summing over rows,
\[
\begin{aligned}
\|Mx\|_2^2
&=\sum_a|(Mx)_a|^2\\
&\leq\sum_a
\left(\sum_b|M_{ab}|\right)
\left(\sum_b|M_{ab}|\,|x_b|^2\right)\\
&\leq
\left(\max_a\sum_b|M_{ab}|\right)
\sum_a\sum_b|M_{ab}|\,|x_b|^2\\
&=
\left(\max_a\sum_b|M_{ab}|\right)
\sum_b|x_b|^2\left(\sum_a|M_{ab}|\right)\\
&\leq
\left(\max_a\sum_b|M_{ab}|\right)
\left(\max_b\sum_a|M_{ab}|\right)\|x\|_2^2.
\end{aligned}
\]
For the Euclidean operator norm, taking the supremum over $x\ne0$ yields
\[
\triplenorm{M}_{\op}^2
=\sup_{x\ne0}\frac{\|Mx\|_2^2}{\|x\|_2^2}
\leq
\left(\max_a\sum_b|M_{ab}|\right)
\left(\max_b\sum_a|M_{ab}|\right).
\]
Since $H_{i,s}$ is symmetric,
\[
\triplenorm{H_{i,s}}_{\op}
\leq\max_a\sum_b|(H_{i,s})_{ab}|\leq\frac12.
\]

\emph{Converting source weights to the Frobenius norm.}
For any matrix $\Delta$, Cauchy--Schwarz and the definition of the
operator norm now give, block by block,
\[
\begin{aligned}
\Delta_{i,(s,\cdot)}^\top H_{i,s}\Delta_{i,(s,\cdot)}
&\leq\|\Delta_{i,(s,\cdot)}\|_2\,
\|H_{i,s}\Delta_{i,(s,\cdot)}\|_2\\
&\leq\triplenorm{H_{i,s}}_{\op}\|\Delta_{i,(s,\cdot)}\|_2^2
\leq\frac12\|\Delta_{i,(s,\cdot)}\|_2^2.
\end{aligned}
\]
Multiply these inequalities by $d_{i,s}$ and take the sum, we get
\begin{align*}
Q(\Delta)
&=\sum_{i,s}d_{i,s}\Delta_{i,(s,\cdot)}^\top H_{i,s}\Delta_{i,(s,\cdot)}\\*
&\leq\frac12\sum_{i,s}d_{i,s}\|\Delta_{i,(s,\cdot)}\|_2^2\\*
&\leq\frac12\left(\max_{i,s}d_{i,s}\right)
\sum_{i,s}\|\Delta_{i,(s,\cdot)}\|_2^2\\*
&=\frac12\left(\max_{i,s}d_{i,s}\right)
\triplenorm{\Delta}_F^2.
\end{align*}

Finally, for each fixed state $s$, nonnegativity of the design probabilities gives
\[
d_{i,s}
\leq\sum_jd_{j,s}
\qquad\text{for every task }i.
\]
Taking maxima over $i,s$, and then using the second term in the maximum
defining $\nu$ in~\eqref{dn:eq:unweighted-variance}, yields
\[
\frac12\max_{i,s}d_{i,s}
\leq\frac12\max_s\sum_id_{i,s}
\leq\nu.
\]
Combining the preceding bounds proves, for every matrix $\Delta$,
\begin{equation}
Q(\Delta)\leq\frac12\left(\max_{i,s}d_{i,s}\right)
\triplenorm{\Delta}_F^2
\leq\nu\triplenorm{\Delta}_F^2.
\label{dn:eq:population-upper}
\end{equation}
Consequently, for every $\Delta\in\mathcal B(t)$,
\[
Q(\Delta)\leq\nu t^2,\qquad
\Var(f_\Delta(I,S))\leq\E f_\Delta(I,S)^2
\leq4B^2 Q(\Delta)\leq4B^2\nu t^2.
\]
In addition, $0\leq Q(\Delta)\leq4B^2$, so 
$|f_\Delta(i,s)-Q(\Delta)|\leq4B^2$.

For $\eta\in\{-1,+1\}$ set
\[
U_\eta:=\sup_{\Delta\in\mathcal B(t)}\eta\{Q_n(\Delta)-Q(\Delta)\}\geq0.
\]
For each fixed sign $\eta$ and $\Delta\in\mathcal B(t)$, under $(I,S)\sim d$ the random variable $\eta\{f_\Delta(I,S)-Q(\Delta)\}$ has mean zero, absolute value at most $4B^2$, and variance at most $4B^2\nu t^2$.
Moreover,
\[
U_\eta=
\sup_{\Delta\in\mathcal B(t)}
\frac1n\sum_{\ell=1}^n\eta\{f_\Delta(I_\ell,S_\ell)-Q(\Delta)\}.
\]
Apply Talagrand's concentration inequality in
Wainwright~\cite[Eq.~(3.86), p.~90]{wainwright2019high}
with $Z=U_\eta$, $b=4B^2$, $\sigma^2\leq4B^2\nu t^2$,
$\epsilon=1$, and tail parameter $x/n$.
Using the constants stated there, for each fixed $\eta\in\{-1,+1\}$ and
$x>0$, with probability at least $1-e^{-x}$,
\begin{equation}
U_\eta\leq2\E U_\eta+2Bt\sqrt{\frac{2\nu x}{n}}+\frac{28B^2x}{3n}.
\label{dn:eq:Bousquet-local}
\end{equation}
For each $x>0$, a union bound over $\eta\in\{-1,+1\}$ yields, with probability at least $1-2e^{-x}$,
\begin{align*}
\sup_{\Delta\in\mathcal B(t)}|Q_n(\Delta)-Q(\Delta)|
&=\max_{\eta\in\{-1,+1\}}U_\eta \leq2\max_{\eta\in\{-1,+1\}}\E U_\eta
+2Bt\sqrt{\frac{2\nu x}{n}}+\frac{28B^2x}{3n}.
\end{align*}

We next control the expectations in~\eqref{dn:eq:Bousquet-local}. Let $\sigma_t$ be independent scalar Rademacher variables and $\epsilon_t\in\{-1,+1\}^A$ independent vectors with independent Rademacher coordinates, all independent of the sample. Let $(I_t',S_t')_{t=1}^n$ be an independent copy of the sample, and let $\E'$ denote expectation over this copy. For each fixed $\eta\in\{-1,+1\}$, since $Q(\Delta)=\E'f_\Delta(I_t',S_t')$,
\begin{align}
\E U_\eta
&=\E\sup_{\Delta\in\mathcal B(t)}
\E'\left[\frac1n\sum_t\eta\{f_\Delta(I_t,S_t)-f_\Delta(I_t',S_t')\}\right]\notag\\
&\leq\E\E'\sup_{\Delta\in\mathcal B(t)}
\frac1n\sum_t\eta\{f_\Delta(I_t,S_t)-f_\Delta(I_t',S_t')\}.
\label{dn:eq:fixed-radius-ghost-sample}
\end{align}
The inequality uses $\sup_\Delta\E'[\,\cdot\,]\leq\E'[\sup_\Delta(\,\cdot\,)]$ with the original sample held fixed.

For each $t$, interchanging $(I_t,S_t)$ and $(I_t',S_t')$ preserves the sampling law and reverses the difference for every $\Delta$ simultaneously. Hence the collection of differences has the same joint distribution after multiplication by the independent signs $\sigma_t$, giving
\begin{align}
&\E\sup_{\Delta\in\mathcal B(t)}\frac1n\sum_t
\eta\{f_\Delta(I_t,S_t)-f_\Delta(I_t',S_t')\}\notag\\*
=&\E\sup_{\Delta\in\mathcal B(t)}\frac1n\sum_t
\eta\sigma_t\{f_\Delta(I_t,S_t)-f_\Delta(I_t',S_t')\}\notag\\*
\leq &2\E\sup_{\Delta\in\mathcal B(t)}\frac1n\sum_t\sigma_t f_\Delta(I_t,S_t).
\label{dn:eq:fixed-radius-symmetrization}
\end{align}

For the final inequality, bound the supremum of the two sums by the sum of their suprema. Their expectations coincide because the samples have the same law and both $\eta\sigma_t$ and $-\eta\sigma_t$ are Rademacher signs.

For every $\Delta\in\mathcal B(t)$ and $(i,s)$,~\eqref{dn:eq:fixed-radius-envelope} gives
\[
\|H_{i,s}^{1/2}\Delta_{i,(s,\cdot)}\|_2^2=f_\Delta(i,s)\leq4B^2.
\]
Consequently, for any $\Delta,\Gamma\in\mathcal B(t)$, Cauchy--Schwarz yields
\[
\begin{aligned}
|f_\Delta(i,s)-f_\Gamma(i,s)|
&\leq\left(\|H_{i,s}^{1/2}\Delta_{i,(s,\cdot)}\|_2
+\|H_{i,s}^{1/2}\Gamma_{i,(s,\cdot)}\|_2\right)
\|H_{i,s}^{1/2}(\Delta-\Gamma)_{i,(s,\cdot)}\|_2\\
&\leq4B\|H_{i,s}^{1/2}(\Delta-\Gamma)_{i,(s,\cdot)}\|_2.
\end{aligned}
\]

Conditional on the sample, apply the vector contraction inequality~\cite[Theorem~3]{Maurer2016} to the scalar-valued functions $\Delta\mapsto f_\Delta(I_\ell,S_\ell)$ and vector-valued functions $\Delta\mapsto4B H_{I_\ell,S_\ell}^{1/2}\Delta_{I_\ell,(S_\ell,\cdot)}$ on $\mathcal B(t)$. This gives

\begin{equation}
\E_\sigma\sup_{\Delta\in\mathcal B(t)} \frac1n\sum_{\ell=1}^n\sigma_\ell f_\Delta(I_\ell,S_\ell) \leq4\sqrt2 B\,\E_\epsilon\sup_{\Delta\in\mathcal B(t)} \frac1n\sum_{\ell=1}^n\langle\epsilon_\ell,
H_{I_\ell,S_\ell}^{1/2}\Delta_{I_\ell,(S_\ell,\cdot)}\rangle.
\label{dn:eq:fixed-radius-vector-contraction}
\end{equation}

Here $\E_\sigma$ and $\E_\epsilon$ denote expectations over the Rademacher variables $(\sigma_\ell)_{\ell=1}^n$ and vectors $(\epsilon_\ell)_{\ell=1}^n$, respectively, with the sample $(I_\ell,S_\ell)_{\ell=1}^n$ held fixed. We denote
\begin{equation}
\Xi_t:=
\sum_a[H_{I_t,S_t}^{1/2}\epsilon_t]_a X(I_t,S_t,a).
\label{dn:eq:Fisher-Rademacher-summand}
\end{equation}

Combining~\eqref{dn:eq:fixed-radius-ghost-sample}--\eqref{dn:eq:fixed-radius-vector-contraction} yields
\begin{align}
\E U_\eta
&\leq8\sqrt2 B\,\E\sup_{\Delta\in\mathcal B(t)}
\frac1n\sum_t\langle\epsilon_t,H_{I_t,S_t}^{1/2}\Delta_{I_t,(S_t,\cdot)}\rangle
\notag\\*
&=8\sqrt2 B\,\E\sup_{\Delta\in\mathcal B(t)}
\left|\traceinner{\frac1n\sum_t\Xi_t}{\Delta}\right|
\notag\\*
&\leq64B\sqrt \rankbound\,t\,\E\triplenorm{\frac1n\sum_t\Xi_t}_{\op}.
\label{dn:eq:fixed-radius-contraction}
\end{align}
The last inequality uses~\eqref{dn:eq:cone-size}. Now fix the observations $(I_\ell,S_\ell)_{\ell=1}^n$ and the Rademacher vectors $(\epsilon_\ell)_{\ell=1}^n$, and define
\[
\mathcal L(\Delta):=
\traceinner{\frac1n\sum_{\ell=1}^n\Xi_\ell}{\Delta}.
\]
The constraints in~\eqref{dn:eq:fixed-radius-class} are unchanged under
$\Delta\mapsto-\Delta$. Thus
\[
\Delta\in\mathcal B(t)\iff-\Delta\in\mathcal B(t),
\qquad \mathcal L(-\Delta)=-\mathcal L(\Delta).
\]
Consequently,
\[
\begin{aligned}
\sup_{\Delta\in\mathcal B(t)}|\mathcal L(\Delta)|
&=\max\left\{
\sup_{\Delta\in\mathcal B(t)}\mathcal L(\Delta),\,
\sup_{\Delta\in\mathcal B(t)}[-\mathcal L(\Delta)]\right\}\\
&=\max\left\{
\sup_{\Delta\in\mathcal B(t)}\mathcal L(\Delta),\,
\sup_{\Delta\in\mathcal B(t)}\mathcal L(-\Delta)\right\}\\
&=\sup_{\Delta\in\mathcal B(t)}\mathcal L(\Delta).
\end{aligned}
\]

We next bound $\E\triplenorm{\frac1n\sum_t\Xi_t}_{\op}$ in~\eqref{dn:eq:fixed-radius-contraction} using matrix Bernstein. Conditional on $(I_t,S_t)=(i,s)$,
\[
\E_\epsilon\Xi_t=0,\qquad
\E_\epsilon[(H_{i,s}^{1/2}\epsilon_t)(H_{i,s}^{1/2}\epsilon_t)^\top]=H_{i,s}.
\]
By~\eqref{dn:eq:Fisher-Rademacher-summand}, the entries of $\Xi_t$ are
\[
[\Xi_t]_{j,(\tilde{s},a)}
=\begin{cases}
[H_{i,s}^{1/2}\epsilon_t]_a,&j=i\ \text{and}\ \tilde{s}=s,\\
0,&\text{otherwise}.
\end{cases}
\]
Thus only row $i$ and columns $(s,1),\ldots,(s,A)$ can contain nonzero entries.
The two matrix products are
\[
\begin{aligned}
[\Xi_t\Xi_t^\top]_{j,k}
&=\sum_{\tilde{s},a}[\Xi_t]_{j,(\tilde{s},a)}[\Xi_t]_{k,(\tilde{s},a)}=\one\{j=i,\ k=i\}\epsilon_t^\top H_{i,s}\epsilon_t,\\
[\Xi_t^\top\Xi_t]_{(\tilde{s},a),(\tilde{s}',b)}
&=\sum_j[\Xi_t]_{j,(\tilde{s},a)}[\Xi_t]_{j,(\tilde{s}',b)} =\one\{\tilde{s}=s,\ \tilde{s}'=s\}[H_{i,s}^{1/2}\epsilon_t]_a[H_{i,s}^{1/2}\epsilon_t]_b.
\end{aligned}
\]
Taking expectations over the signs uses
\[
\E_\epsilon[\epsilon_t^\top H_{i,s}\epsilon_t]=\tr(H_{i,s}),\qquad
\E_\epsilon\!\left[[H_{i,s}^{1/2}\epsilon_t]_a[H_{i,s}^{1/2}\epsilon_t]_b\right]=(H_{i,s})_{a,b}.
\]
Using the basis vectors $e_i,f_{(s,a)}$ from~\eqref{dn:eq:basis-matrix} and
$\Pp\{(I_t,S_t)=(i,s)\}=d_{i,s}$, we obtain
\begin{align*}
\E[\Xi_t\Xi_t^\top]
&=\sum_{i,s}d_{i,s}\tr(H_{i,s})e_i e_i^\top =\operatorname{diag}_{i=1}^N\left(\sum_sd_{i,s}\tr(H_{i,s})\right),\\
\E[\Xi_t^\top\Xi_t]
&=\sum_{i,s}d_{i,s}
\sum_{a,b}(H_{i,s})_{a,b}f_{(s,a)}f_{(s,b)}^\top =\operatorname{blockdiag}_{s=1}^S\left(\sum_id_{i,s}H_{i,s}\right).
\end{align*}

Using $\tr(H_{i,s})=1-\sum_a\pi_{1,i}^\star(a\mid s)^2\leq1-1/A$ gives
\[
\begin{aligned}
\triplenorm{\E[\Xi_t\Xi_t^\top]}_{\op}
&=\max_i\sum_sd_{i,s}\tr(H_{i,s}) \leq\left(1-\frac1A\right)\max_i\sum_sd_{i,s}
\leq\nu.
\end{aligned}
\]
On the other hand,
\[
\begin{aligned}
\triplenorm{\E[\Xi_t^\top\Xi_t]}_{\op}
&=\max_s\triplenorm{\sum_id_{i,s}H_{i,s}}_{\op} \leq\max_s\sum_id_{i,s}\triplenorm{H_{i,s}}_{\op} \leq\frac12\max_s\sum_id_{i,s}
\leq\nu.
\end{aligned}
\]
Both final inequalities use the definition of $\nu$ in~\eqref{dn:eq:unweighted-variance}.
Also,
\begin{align*}
\triplenorm{\Xi_t}_{\op}^2
&=\epsilon_t^\top H_{I_t,S_t}\epsilon_t =1-\left(\sum_a\pi_{1,I_t}^\star(a\mid S_t)\epsilon_{t,a}\right)^2
\leq1\leq 2.
\end{align*}
Thus each summand $\Xi_t$ is independent, mean zero, bounded by $\sqrt2$, with sum variance statistic at most $n\nu$. The expectation part of rectangular matrix Bernstein~\cite[Theorem~6.1.1]{Tropp2015} gives
\begin{equation}
\E\triplenorm{\frac1n\sum_t\Xi_t}_{\op}
\leq\sqrt{\frac{2\nu\log D}{n}}+\frac{\sqrt2\log D}{3n}.
\label{dn:eq:Fisher-process-expectation}
\end{equation}

Combining~\eqref{dn:eq:Bousquet-local}--\eqref{dn:eq:Fisher-process-expectation}, with probability at least $1-2e^{-x}$,
\begin{equation}
\sup_{\Delta\in\mathcal B(t)}|Q_n(\Delta)-Q(\Delta)|
\leq128B\sqrt \rankbound\,t\left[\sqrt{\frac{2\nu\log D}{n}}+\frac{\sqrt2\log D}{3n}\right]
+2Bt\sqrt{\frac{2\nu x}{n}}+\frac{28B^2x}{3n}.
\label{dn:eq:fixed-radius-high-probability}
\end{equation}
Set $x=\log(4/\delta)$, so $1-2e^{-x}=1-\delta/2$.
Since $D=N+SA\geq 2$, we have $x=\log\frac4\delta\leq\log\frac{2D}{\delta}$ and $\log D\leq\log\frac{2D}{\delta}$. The definition of $\lambda_n$ in~\eqref{dn:eq:unweighted-lambda} therefore gives
\[
\begin{aligned}
\sqrt{\frac{2\nu\log D}{n}}+\frac{\sqrt2\log D}{3n}
&\leq\frac{\lambda_n}2,\\
\sqrt{\frac{2\nu x}{n}}
&\leq\frac{\lambda_n}2.
\end{aligned}
\]
Using $\rankbound\geq1$, the first two terms in~\eqref{dn:eq:fixed-radius-high-probability} are thus bounded by
\[
64B\sqrt \rankbound\,\lambda_nt+B\lambda_nt
\leq65B\sqrt \rankbound\,\lambda_nt.
\]
Therefore, for the fixed deterministic radius $t>0$, with probability at least $1-\delta/2$,
\begin{equation}
\sup_{\Delta\in\mathcal B(t)}|Q_n(\Delta)-Q(\Delta)|
\leq65B\sqrt \rankbound\lambda_nt+\frac{28B^2}{3n}\log\frac{4}{\delta}.
\label{dn:eq:localized-deviation}
\end{equation}
This is a statement at the one fixed deterministic $t$. It is not a simultaneous statement over radii.

\paragraph{Step 5: empirical curvature beyond a deterministic radius.}
We first derive a sufficient lower bound on a deterministic radius $t$, then choose $t_0$ accordingly. Fix $t>0$ and consider any $\Gamma$ satisfying
\[
\Gamma\in\mathcal B(t),
\qquad \triplenorm{\Gamma}_F=t.
\]
By~\eqref{dn:eq:general-curvature},
\[
Q(\Gamma)\geq\frac{e^{-2B}c_0}{A^2}\triplenorm{\Gamma}_F^2
=\frac{e^{-2B}c_0}{A^2}t^2>0.
\]
On the fixed-radius event~\eqref{dn:eq:localized-deviation}, with $x=\log(4/\delta)$,
\[
\begin{aligned}
\frac{|Q_n(\Gamma)-Q(\Gamma)|}{Q(\Gamma)}
&\leq\frac{65B\sqrt \rankbound\,\lambda_nt+28B^2x/(3n)}{Q(\Gamma)}
&&\text{by~\eqref{dn:eq:localized-deviation}},\\
&\leq\frac{A^2e^{2B}\{65B\sqrt \rankbound\,\lambda_nt+28B^2x/(3n)\}}{c_0t^2}
&&\text{by~\eqref{dn:eq:general-curvature} and }\triplenorm{\Gamma}_F=t,\\
&=\frac{65A^2e^{2B}B\sqrt \rankbound\,\lambda_n}{c_0t}
+\frac{28A^2e^{2B}B^2x}{3nc_0t^2}.
\end{aligned}
\]

Setting each term to be at most $1/4$ gives
\[
\begin{aligned}
\frac{65A^2e^{2B}B\sqrt \rankbound\,\lambda_n}{c_0t}\leq\frac14
&\iff t\geq\frac{260A^2e^{2B}B\sqrt \rankbound\,\lambda_n}{c_0},\\
\frac{28A^2e^{2B}B^2x}{3nc_0t^2}\leq\frac14
&\iff t\geq\sqrt{\frac{112A^2e^{2B}B^2x}{3nc_0}}.
\end{aligned}
\]

For any $\Gamma\in\mathcal C(\Theta^\star)\setminus\{0\}$, equations~\eqref{dn:eq:general-curvature} and~\eqref{dn:eq:population-upper} imply
\[
\frac{e^{-2B}c_0}{A^2}\triplenorm{\Gamma}_F^2
\leq Q(\Gamma)\leq\nu\triplenorm{\Gamma}_F^2
\quad\Longrightarrow\quad \frac{e^{-2B}c_0}{A^2}\leq\nu.
\]
The definition of $\lambda_n$ in~\eqref{dn:eq:unweighted-lambda}, $\rankbound\geq1$, and $x=\log(4/\delta)\leq\log(2(N+SA)/\delta)$ give
\[
\rankbound\lambda_n^2
\geq\frac{8\rankbound\nu}{n}\log\frac{2(N+SA)}{\delta}
\geq\frac{8e^{-2B}c_0x}{A^2n}.
\]
Consequently, the second lower bound on $t$ satisfies
\[
\sqrt{\frac{112A^2e^{2B}B^2x}{3nc_0}}
\leq\sqrt{\frac{14}{3}}\,
\frac{A^2e^{2B}B\sqrt \rankbound\,\lambda_n}{c_0}.
\]
Since $\sqrt{14/3}<260$, the first condition implies the second. We therefore impose
\[
t\geq\frac{260 A^2e^{2B}B\sqrt \rankbound\,\lambda_n}{c_0}.
\]

We also require that $t$ bound the estimation error $\Delta=\widehat\Theta-\Theta^\star$.
On the good event~\eqref{dn:eq:good-event}, once $Q_n(\Delta)\geq Q(\Delta)/2$ has been established for $\triplenorm{\Delta}_F\geq t$, the bounds~\eqref{dn:eq:general-curvature}, \eqref{dn:eq:Bregman-lower}, and~\eqref{dn:eq:Bregman-upper} give
\[
\frac{h_Be^{-2B}c_0}{2A^2}\triplenorm{\Delta}_F^2
\leq D_{L_n}(\Theta^\star+\Delta,\Theta^\star)
\leq\frac{3\sqrt2}{2}\sqrt \rankbound\,\lambda_n\triplenorm{\Delta}_F.
\]
For $\triplenorm{\Delta}_F\geq t>0$, division by $\triplenorm{\Delta}_F$ shows that the additional requirement
\[
t\geq\frac{3\sqrt2 A^2e^{2B}\sqrt \rankbound\,\lambda_n}{h_Bc_0}
\]
ensures $\triplenorm{\Delta}_F\leq t$. The case $\triplenorm{\Delta}_F<t$ already satisfies this bound.
Since $3\sqrt2\leq6$, the following radius satisfies both requirements:
\begin{equation}
t_0:=\frac{A^2e^{2B}(260B+6/h_B)\sqrt \rankbound\,\lambda_n}{c_0}.
\label{dn:eq:localization-radius}
\end{equation}
Apply~\eqref{dn:eq:localized-deviation} once with $t=t_0$; its event holds uniformly over $\mathcal B(t_0)$.

For any $\Gamma\in\mathcal C(\Theta^\star)$ with $\norm{\Gamma}_\infty\leq2B$ and $\triplenorm{\Gamma}_F\geq t_0$, scale to that radius:
\[
\Gamma_{t_0}:=\frac{t_0}{\triplenorm{\Gamma}_F}\Gamma
\in\mathcal C(\Theta^\star),\qquad
\triplenorm{\Gamma_{t_0}}_F=t_0,\qquad
\norm{\Gamma_{t_0}}_\infty\leq\norm{\Gamma}_\infty\leq2B.
\]
Thus $\Gamma_{t_0}\in\mathcal B(t_0)$.
We evaluate the Step~4 bound at this scaled matrix $\Gamma_{t_0}$:
\[
\begin{aligned}
|Q_n(\Gamma_{t_0})-Q(\Gamma_{t_0})|
&\leq\sup_{\Delta\in\mathcal B(t_0)}|Q_n(\Delta)-Q(\Delta)|\\
&\leq65B\sqrt \rankbound\,\lambda_nt_0
+\frac{28B^2}{3n}\log\frac{4}{\delta}
\qquad\text{by~\eqref{dn:eq:localized-deviation}}.
\end{aligned}
\]
The empirical-process bound holds simultaneously for every $\Delta\in\mathcal B(t_0)$ on this event.
Also,~\eqref{dn:eq:general-curvature} gives $Q(\Gamma_{t_0})\geq\frac{e^{-2B}c_0}{A^2} t_0^2$.
To transfer the resulting relative-error bound to the original matrix $\Gamma$, use quadratic homogeneity:
\[
Q_n(\Gamma_{t_0})=\frac{t_0^2}{\triplenorm{\Gamma}_F^2}Q_n(\Gamma),
\qquad
Q(\Gamma_{t_0})=\frac{t_0^2}{\triplenorm{\Gamma}_F^2}Q(\Gamma).
\]
Consequently, on the same fixed-radius event,
\[
\frac{|Q_n(\Gamma)-Q(\Gamma)|}{Q(\Gamma)}
=\frac{|Q_n(\Gamma_{t_0})-Q(\Gamma_{t_0})|}{Q(\Gamma_{t_0})}
\leq\frac14+\frac14=\frac12.
\]
We obtain the empirical curvature bound
\begin{equation}
Q_n(\Gamma)\geq\frac12Q(\Gamma)
\geq\frac{e^{-2B}c_0}{2A^2}\triplenorm{\Gamma}_F^2,
\qquad
\Gamma\in\mathcal C(\Theta^\star),\quad
\norm{\Gamma}_\infty\leq2B,\quad
\triplenorm{\Gamma}_F\geq t_0.
\label{dn:eq:localized-empirical-curvature}
\end{equation}
On the event in~\eqref{dn:eq:localized-deviation} with $t=t_0$, which has probability at least $1-\delta/2$, inequality~\eqref{dn:eq:localized-empirical-curvature} holds simultaneously for every $\Gamma\in\mathcal C(\Theta^\star)$ satisfying $\norm{\Gamma}_\infty\leq2B$ and $\triplenorm{\Gamma}_F\geq t_0$.

\paragraph{Step 6: apply optimality and obtain the error bound.}
On the intersection of~\eqref{dn:eq:good-event} and~\eqref{dn:eq:localized-deviation} at $t=t_0$, fix any minimizer $\widehat\Theta$ and set $\Delta=\widehat\Theta-\Theta^\star$.
If $\triplenorm{\Delta}_F<t_0$, the desired bound already holds by the choice~\eqref{dn:eq:localization-radius} of $t_0$.
Otherwise,~\eqref{dn:eq:cone-condition} and feasibility give $\Delta\in\mathcal C(\Theta^\star)$ and $\norm{\Delta}_\infty\leq2B$. Combining~\eqref{dn:eq:localized-empirical-curvature} with the Bregman bounds~\eqref{dn:eq:Bregman-lower} and~\eqref{dn:eq:Bregman-upper} yields
\[
\frac{h_Be^{-2B}c_0}{2A^2}\triplenorm{\Delta}_F^2
\leq h_B Q_n(\Delta)
\leq D_{L_n}(\Theta^\star+\Delta,\Theta^\star)
\leq\frac{3\sqrt2}{2}\sqrt \rankbound\,\lambda_n
\triplenorm{\Delta}_F.
\]
Dividing by $h_Be^{-2B}c_0\triplenorm{\Delta}_F/(2A^2)>0$ and using~\eqref{dn:eq:localization-radius}, we get
\[
\triplenorm{\Delta}_F
\leq\frac{3\sqrt2 A^2e^{2B}\sqrt \rankbound\,\lambda_n}{h_Bc_0}
\leq t_0.
\]

We next verify that the good event and the fixed-radius event hold simultaneously with probability at least $1-\delta$.

Denote the good event~\eqref{dn:eq:good-event} by
\[
\mathcal E_{\mathrm{good}}
:=
\left\{
\triplenorm{\nabla L_n(\Theta^\star)}_{\op}
\leq\frac{\lambda_n}2
\right\}.
\]
By~\eqref{dn:eq:matrix-Bernstein-score} and~\eqref{dn:eq:unweighted-lambda},
\[
\Pp(\mathcal E_{\mathrm{good}}^c)
\leq
D\exp\!\left[-\log\frac{2(N+SA)}{\delta}\right]
=\frac{\delta}{2}.
\]
Use the deterministic $t_0$ from~\eqref{dn:eq:localization-radius} and define
\[
x:=\log\frac4\delta,\qquad
\mathcal E_{\mathrm{loc}}
:=\left\{
\sup_{\Gamma\in\mathcal B(t_0)}
|Q_n(\Gamma)-Q(\Gamma)|
\leq65B\sqrt \rankbound\,\lambda_nt_0+\frac{28B^2x}{3n}
\right\}.
\]
Equations~\eqref{dn:eq:fixed-radius-high-probability}--\eqref{dn:eq:localized-deviation}, at $t=t_0$, give
\[
\Pp(\mathcal E_{\mathrm{loc}}^c)\leq2e^{-x}
=2\frac{\delta}{4}=\frac{\delta}{2}.
\]

The union bound gives
\[
\Pp(\mathcal E_{\mathrm{good}}\cap\mathcal E_{\mathrm{loc}})
\geq1-\Pp(\mathcal E_{\mathrm{good}}^c)-\Pp(\mathcal E_{\mathrm{loc}}^c)
\geq1-\delta.
\]

Write the full set of minimizers in~\eqref{dn:eq:unweighted-estimator} as
\[
\widehat{\mathcal M}
:=\argmin_{\Theta\in\mathcal M_B}
\left\{L_n(\Theta)+\lambda_n\triplenorm{\Theta}_*\right\}.
\]
On $\mathcal E_{\mathrm{good}}\cap\mathcal E_{\mathrm{loc}}$, equations~\eqref{dn:eq:cone-condition} and~\eqref{dn:eq:Bregman-upper} give, for every $\widehat\Theta\in\widehat{\mathcal M}$,
\[
\Delta:=\widehat\Theta-\Theta^\star\in\mathcal C(\Theta^\star),
\qquad
D_{L_n}(\Theta^\star+\Delta,\Theta^\star)
\leq\frac{3\lambda_n}2\sqrt{2\rankbound}\,\triplenorm{\Delta}_F.
\]
Because $\Theta^\star\neq0$ and $\Theta^\star\in T$, the cone is nontrivial, so Step~6 applies to each such minimizer and gives $\triplenorm{\Delta}_F\leq t_0$. Thus, using~\eqref{dn:eq:localization-radius} and~\eqref{dn:eq:unweighted-logit-radius}, on $\mathcal E_{\mathrm{good}}\cap\mathcal E_{\mathrm{loc}}$,
\[
\sup_{\widehat\Theta\in\widehat{\mathcal M}}
\triplenorm{\widehat\Theta-\Theta^\star}_F
\leq\varepsilon_n
=\frac{A^2e^{2B}(260B+6/h_B)\sqrt \rankbound\,\lambda_n}{c_0}.
\]
Consequently,
\[
\Pp\left[
\sup_{\widehat\Theta\in\widehat{\mathcal M}}
\triplenorm{\widehat\Theta-\Theta^\star}_F
\leq\varepsilon_n
\right]
\geq\Pp(\mathcal E_{\mathrm{good}}\cap\mathcal E_{\mathrm{loc}})\geq1-\delta,
\]
which is~\eqref{dn:eq:general-error-bound}.
\end{proof}

\subsection{Reward identification and value transfer}
\label{dn:app:transfer}

Section~\ref{dn:app:estimation} bounds the error in the fitted source logits. To turn this into a target regret bound, we first identify the rewards and then bound the resulting errors in target-policy values.
Under the normalization in Assumption~\ref{dn:ass:normalization}, each task's source log-policy uniquely determines its reward (Lemma~\ref{dn:lm:irl_reward}). The Lipschitz log-softmax map and the affine reward-recovery formula thus propagate logit error into reward error. For a fixed target policy, the value error is the expected discounted sum of reward errors along its trajectories. Lemma~\ref{dn:lem:general-value} combines these relations to obtain simultaneous value-error bounds for all tasks and candidate policies.

Lemma~\ref{dn:lem:general-selection} in Section~\ref{dn:app:selection} applies these simultaneous value-error bounds to plug-in selection. It bounds the value loss relative to a comparator by the sum of that comparator's radius and the selected policy's radius. Taking the comparator to be $\pi_{2,i}^\star$ gives the regret bound. 

\subsubsection{Reward recovery under the normalization}

\begin{lemma}[Normalized reward is affine in the log-policy]
\label{dn:lm:irl_reward}
Under Assumption~\ref{dn:ass:normalization}, let
$u_i^\star(s,a):=\log\pi_{1,i}^\star(a\mid s)$ and
\[
L_\mu:=C_\mu(I-\gamma_1\Transition_1\mu)^{-1}.
\]
The unique reward inducing the source soft-optimal policy
$\pi_{1,i}^\star$ and satisfying
\eqref{dn:eq:reward-normalization} is
\[
r_i=L_\mu(u_i^\star-g)+g
   =L_\mu u_i^\star+(I-L_\mu)g.
\]
Thus, for fixed source dynamics, discount, and normalization
$(\mu,g)$, reward recovery is affine in the source log-policy.
\end{lemma}

\begin{proof}
Fix task $i$.  First determine all reward--value pairs $(r_i,V_{1,i})$ consistent with the source policy $\pi_{1,i}^\star$.  The unit-temperature soft Bellman equations are
\begin{align}
Q_{1,i}&=r_i+\gamma_1\Transition_1V_{1,i},\notag\\*
V_{1,i}(s)&=\log\sum_a e^{Q_{1,i}(s,a)},
\qquad
\pi_{1,i}^\star(a\mid s)=e^{Q_{1,i}(s,a)-V_{1,i}(s)}.
\label{dn:eq:source-soft-Bellman}
\end{align}
Taking logarithms of the policy equation gives $Q_{1,i}=u_i^\star+V_{1,i}$.  Substituting into the first Bellman equation shows that every solution belongs to the family
\[
\left\{\bigl(u_i^\star+V-\gamma_1\Transition_1V,\,V\bigr):V\in\R^S\right\}.
\]
Conversely, choose any state function $V$, set $r_i=u_i^\star+V-\gamma_1\Transition_1V$, and let $Q_{1,i}=u_i^\star+V$.  Then $Q_{1,i}=r_i+\gamma_1\Transition_1V$, and $\sum_a e^{u_i^\star(s,a)}=1$ gives
\begin{align*}
\log\sum_a e^{Q_{1,i}(s,a)}
&=V(s)+\log\sum_a e^{u_i^\star(s,a)}=V(s),\\*
\softmax(Q_{1,i}(s,\cdot))_a
&=e^{u_i^\star(s,a)}=\pi_{1,i}^\star(a\mid s).
\end{align*}
Thus every pair in the displayed family solves the soft Bellman equations for the observed policy.

Now impose $\mu r_i=g$.  Within this family, the constraint is
\begin{align*}
g=\mu r_i&=\mu u_i^\star+V_{1,i}-\gamma_1\mu\Transition_1V_{1,i}
\end{align*}
Rearranging the preceding equality gives
\begin{align*}
(I-\gamma_1\mu\Transition_1)V_{1,i}&=g-\mu u_i^\star.
\end{align*}
Since $\mu\Transition_1$ is a stochastic state-to-state kernel and $\gamma_1<1$, this determines a unique state value
\[
V_{1,i}=(I-\gamma_1\mu\Transition_1)^{-1}(g-\mu u_i^\star),
\]
and hence a unique reward $r_i=u_i^\star+V_{1,i}-\gamma_1\Transition_1V_{1,i}$.

Substituting this value of $V_{1,i}$ directly gives
\[
r_i
=u_i^\star+(I-\gamma_1\mu\Transition_1)^{-1}(g-\mu u_i^\star)
-\gamma_1\Transition_1(I-\gamma_1\mu\Transition_1)^{-1}(g-\mu u_i^\star).
\]
Since $V_{1,i}$ and $g$ depend only on the state and $\sum_a\mu(a\mid s)=1$, averaging over actions gives $\mu V_{1,i}=V_{1,i}$ and $\mu g=g$. The definition~\eqref{dn:eq:C-mu} of $C_\mu$ and linearity of $\mu$ then give
\[
\begin{aligned}
C_\mu V_{1,i}&=V_{1,i}-\mu V_{1,i}=0,
\qquad C_\mu g=g-\mu g=0,\\
g-\mu u_i^\star&=\mu g-\mu u_i^\star=-\mu(u_i^\star-g).
\end{aligned}
\]
Using these identities and $\mu r_i=g$, we obtain
\begin{align*}
r_i
&=g+C_\mu r_i
=g+C_\mu(u_i^\star-g-\gamma_1\Transition_1V_{1,i})\\*
&=g+C_\mu\left[(u_i^\star-g)
+\gamma_1\Transition_1(I-\gamma_1\mu\Transition_1)^{-1}\mu(u_i^\star-g)\right]\\*
&=g+C_\mu\left[I+\gamma_1\Transition_1(I-\gamma_1\mu\Transition_1)^{-1}\mu\right](u_i^\star-g).
\end{align*}

The operator in brackets equals $(I-\gamma_1\Transition_1\mu)^{-1}$ since direct multiplication gives
\begin{align*}
&(I-\gamma_1\Transition_1\mu)
\left[I+\gamma_1\Transition_1(I-\gamma_1\mu\Transition_1)^{-1}\mu\right]\\
&=I-\gamma_1\Transition_1\mu
+\gamma_1\Transition_1(I-\gamma_1\mu\Transition_1)(I-\gamma_1\mu\Transition_1)^{-1}\mu\\
&=I.
\end{align*}
Therefore
\[
r_i=C_\mu(I-\gamma_1\Transition_1\mu)^{-1}(u_i^\star-g)+g.
\]
Since $L_\mu$ is linear and $(I-L_\mu)g$ does not depend on $u_i^\star$, the recovery map is affine in $u_i^\star$.
\end{proof}

\subsubsection{From fitted logits to target-policy values}
The following lemma converts the logit-error bound into simultaneous bounds on the value errors of all candidate policies.

\begin{lemma} \label{dn:lem:general-value}
Suppose Assumption~\ref{dn:ass:normalization} holds and $\Theta^\star\in\mathcal M_B$. For any feasible matrix $\widehat\Theta\in\mathcal M_B$, construct $\widehat u_i$, $\widehat r_i$, and $\widehat \softvalue_{2,i}$ by~\eqref{dn:eq:unweighted-U-hat}--\eqref{dn:eq:unweighted-J-hat}.
For any $\epsilon\geq0$, the implication
\begin{equation}
\begin{aligned}
&\triplenorm{\widehat\Theta-\Theta^\star}_F\leq \epsilon\\
&\quad\Longrightarrow\quad
\left|\widehat \softvalue_{2,i}(\pi)-\softvalue_{2,i}(\pi;r_i)\right| \leq\frac{L_B \epsilon}{1-\gamma_2}
\left\{\sum_s
\left\|\left[
\bigl[C_\mu(I-\gamma_1\Transition_1\mu)^{-1}\bigr]^\top
\rho_2^\pi\right]_{(s,\cdot)}\right\|_2^2
\right\}^{1/2}
\end{aligned}
\label{dn:eq:general-value-certificate}
\end{equation}
holds simultaneously for every task $i$ and candidate $\pi\in\Pi$.
Here $L_B$ is the constant in~\eqref{dn:eq:kappa-B}.
\end{lemma}

\begin{proof}
Fix $i$ and $\pi$ and write
\[
\Delta:=\widehat\Theta-\Theta^\star.
\]
\paragraph{Step 1: the blockwise log-softmax difference.}
For each state $s$, the fundamental theorem of calculus applied to
log-sum-exp gives
\begin{align*}
&\log\sum_b e^{\widehat\Theta_{i,(s,b)}}
-\log\sum_b e^{\Theta^\star_{i,(s,b)}}=\int_0^1\sum_b
\softmax(\Theta^\star_{i,(s,\cdot)}+t\Delta_{i,(s,\cdot)})_b
\Delta_{i,(s,b)}\,dt\\
&=\bar\pi_{i,(s,\cdot)}^\top\Delta_{i,(s,\cdot)},
\end{align*}
where we denote
\[
\bar\pi_{i,(s,\cdot)}:=\int_0^1
\softmax(\Theta^\star_{i,(s,\cdot)}+t\Delta_{i,(s,\cdot)})\,dt.
\]
The vector $\bar\pi_{i,(s,\cdot)}$ is a probability vector. Subtracting the
log-normalizing constants from the logits therefore yields
\begin{equation}
\widehat u_{i,(s,\cdot)}-u_i^\star(s,\cdot)
=(I_A-\one\bar\pi_{i,(s,\cdot)}^\top)\Delta_{i,(s,\cdot)}.
\label{dn:eq:logsoftmax-path-difference}
\end{equation}

\paragraph{Step 2: the log-softmax operator bound.}
We now proceed to bound the operator norm of the matrix in~\eqref{dn:eq:logsoftmax-path-difference}. Fix a state $s$. For any $x\in\R^A$, expanding the squared norm and completing the square gives
\begin{align*}
\|(I_A-\one\bar\pi_{i,(s,\cdot)}^\top)x\|_2^2
&=\|x\|_2^2-2(\one^\top x)(\bar\pi_{i,(s,\cdot)}^\top x)
+A(\bar\pi_{i,(s,\cdot)}^\top x)^2\\
&=\|x\|_2^2-\frac{(\one^\top x)^2}{A}
+A\left[\left(\bar\pi_{i,(s,\cdot)}-\frac1A\one\right)^\top x\right]^2.
\end{align*}
Dropping the nonpositive term $-(\one^\top x)^2/A$ and applying Cauchy--Schwarz yields
\begin{align*}
\|(I_A-\one\bar\pi_{i,(s,\cdot)}^\top)x\|_2^2
&\leq\left(1+A\left\|\bar\pi_{i,(s,\cdot)}-\frac1A\one\right\|_2^2\right)\|x\|_2^2=A\|\bar\pi_{i,(s,\cdot)}\|_2^2\|x\|_2^2,
\end{align*}
where the last equality uses $\one^\top\bar\pi_{i,(s,\cdot)}=1$. Taking the supremum over $\|x\|_2=1$ gives
\begin{equation}
\triplenorm{I_A-\one\bar\pi_{i,(s,\cdot)}^\top}_{\op}
\leq\sqrt A\|\bar\pi_{i,(s,\cdot)}\|_2.
\label{dn:eq:logsoftmax-Jacobian}
\end{equation}

Since $\Theta^\star,\widehat\Theta\in\mathcal M_B$, every coordinate of $\Theta^\star+t\Delta=(1-t)\Theta^\star+t\widehat\Theta$ lies in $[-B,B]$ for all $t\in[0,1]$. 
\[
\softmax(z)_a
=\frac{1}{1+\sum_{b\ne a}e^{z_b-z_a}}
\leq\frac{1}{1+(A-1)e^{-2B}}
=\frac{e^{2B}}{e^{2B}+A-1}.
\]
The same bound holds for every coordinate of $\bar\pi_{i,(s,\cdot)}$. Since
$\|\bar\pi_{i,(s,\cdot)}\|_2^2\leq\|\bar\pi_{i,(s,\cdot)}\|_\infty
\sum_a\bar\pi_{i,(s,a)}=\|\bar\pi_{i,(s,\cdot)}\|_\infty$, applying~\eqref{dn:eq:logsoftmax-Jacobian} gives
\begin{equation}
\triplenorm{I_A-\one\bar\pi_{i,(s,\cdot)}^\top}_{\op}
\leq\sqrt{\frac{Ae^{2B}}{e^{2B}+A-1}}=L_B.
\label{dn:eq:Jacobian-bound}
\end{equation}
Combining~\eqref{dn:eq:logsoftmax-path-difference} and~\eqref{dn:eq:Jacobian-bound},
\begin{equation}
\|\widehat u_{i,(s,\cdot)}-u_i^\star(s,\cdot)\|_2
\leq L_B\|\Delta_{i,(s,\cdot)}\|_2.
\label{dn:eq:logsoftmax-block-bound}
\end{equation}
In particular, for each fixed task $i$,
\begin{align}
\sum_s 
\|\widehat u_{i,(s,\cdot)}-u_i^\star(s,\cdot)\|_2^2
&\leq L_B^2\sum_s \|\Delta_{i,(s,\cdot)}\|_2^2\notag\\*
&\leq L_B^2\sum_j\sum_s \|\Delta_{j,(s,\cdot)}\|_2^2
=L_B^2\triplenorm\Delta_F^2.
\label{dn:eq:logprob-weighted-bound}
\end{align}

\paragraph{Step 3: reward recovery and the occupancy identity.}
Subtracting~\eqref{dn:eq:normalized-reward-recovery} from
\eqref{dn:eq:unweighted-r-hat} gives
\begin{equation}
\widehat r_i-r_i
=C_\mu(I-\gamma_1\Transition_1\mu)^{-1}
(\widehat u_i-u_i^\star).
\label{dn:eq:general-reward-difference}
\end{equation}
For a fixed target policy, its entropy term is the same for both rewards, so it cancels in their value difference. By the occupancy identity~\eqref{dn:eq:occupancy-identity}, we therefore have
\begin{align}
\widehat \softvalue_{2,i}(\pi)-\softvalue_{2,i}(\pi;r_i)
&=\frac{1}{1-\gamma_2}\sum_{s,a}\rho_2^\pi(s,a)
\{\widehat r_i(s,a)-r_i(s,a)\}\notag\\*
&=\frac{1}{1-\gamma_2}
\left\langle
\bigl[C_\mu(I-\gamma_1\Transition_1\mu)^{-1}\bigr]^\top\rho_2^\pi,
\widehat u_i-u_i^\star
\right\rangle.
\label{dn:eq:general-value-difference}
\end{align}

We upper bound the inner product by
\begin{align*}
&\left|\left\langle
\bigl[C_\mu(I-\gamma_1\Transition_1\mu)^{-1}\bigr]^\top\rho_2^\pi,
\widehat u_i-u_i^\star\right\rangle\right|\\
&=\left|\sum_s\left\langle
\left[\bigl[C_\mu(I-\gamma_1\Transition_1\mu)^{-1}\bigr]^\top\rho_2^\pi\right]_{(s,\cdot)},
\widehat u_{i,(s,\cdot)}-u_i^\star(s,\cdot)
\right\rangle\right|\\
&\leq\sum_s
\left\|\left[\bigl[C_\mu(I-\gamma_1\Transition_1\mu)^{-1}\bigr]^\top
\rho_2^\pi\right]_{(s,\cdot)}\right\|_2
\|\widehat u_{i,(s,\cdot)}-u_i^\star(s,\cdot)\|_2\\
&\leq
\left\{\sum_s\left\|\left[
\bigl[C_\mu(I-\gamma_1\Transition_1\mu)^{-1}\bigr]^\top
\rho_2^\pi\right]_{(s,\cdot)}\right\|_2^2\right\}^{1/2}
\left\{\sum_s
\|\widehat u_{i,(s,\cdot)}-u_i^\star(s,\cdot)\|_2^2
\right\}^{1/2}\\
&\leq L_B\triplenorm\Delta_F
\left\{\sum_s\left\|\left[
\bigl[C_\mu(I-\gamma_1\Transition_1\mu)^{-1}\bigr]^\top
\rho_2^\pi\right]_{(s,\cdot)}\right\|_2^2\right\}^{1/2}.
\end{align*}

The first equality decomposes the inner product by state. The first inequality uses the triangle inequality and Cauchy--Schwarz on the actions. The second inequality is Cauchy--Schwarz over $s$, and the last inequality follows by ~\eqref{dn:eq:logprob-weighted-bound}. Now use $\triplenorm\Delta_F\leq \epsilon$ and divide by $1-\gamma_2$ to obtain~\eqref{dn:eq:general-value-certificate}. 
\end{proof}

\subsection{Plug-in selection and soft regret}
\label{dn:app:selection}

We apply the simultaneous value-error bounds from Section~\ref{dn:app:transfer} to plug-in selection, then use the soft regret--KL identity to bound policy error.

\subsubsection{Selection from simultaneous value certificates}
The following lemma bounds the loss from maximizing estimated values by the sum of the comparator's and selected policy's error radii.

\begin{lemma}
\label{dn:lem:general-selection}
Fix a task $i$. Suppose $\widetilde \softvalue_{2,i}:\Pi\to\R$ and nonnegative radii $R_i(\pi)$ satisfy
\begin{equation}
|\widetilde \softvalue_{2,i}(\pi)-\softvalue_{2,i}(\pi;r_i)|\leq R_i(\pi)
\qquad\forall\pi\in\Pi.
\label{dn:eq:selection-value-certificates}
\end{equation}
Then for any $\widetilde\pi_i$ such that 
\[
\widetilde\pi_i\in\argmax_{\pi\in\Pi}\widetilde \softvalue_{2,i}(\pi),
\]
every $\pi\in\Pi$ satisfies
\begin{equation}
\softvalue_{2,i}(\pi;r_i)-\softvalue_{2,i}(\widetilde\pi_i;r_i)
\leq R_i(\pi)+R_i(\widetilde\pi_i).
\label{dn:eq:general-comparator-bound}
\end{equation}
\end{lemma}

\begin{proof}
Fix a comparator $\pi\in\Pi$, then we have
\begin{align*}
\softvalue_{2,i}(\pi;r_i)
&\leq\widetilde \softvalue_{2,i}(\pi)+R_i(\pi)\\*
&\leq\widetilde \softvalue_{2,i}(\widetilde\pi_i)+R_i(\pi)\\*
&\leq \softvalue_{2,i}(\widetilde\pi_i;r_i)+R_i(\widetilde\pi_i)+R_i(\pi).
\end{align*}
The first and last inequalities use~\eqref{dn:eq:selection-value-certificates} for $\pi$ and $\widetilde\pi_i$, respectively. The middle inequality holds because $\widetilde\pi_i$ maximizes $\widetilde \softvalue_{2,i}$. Subtracting $\softvalue_{2,i}(\widetilde\pi_i;r_i)$ proves~\eqref{dn:eq:general-comparator-bound}.
\end{proof}

\subsubsection{Soft regret as KL divergence}
Take $\pi=\pi_{2,i}^\star$ in~\eqref{dn:eq:general-comparator-bound} and apply the following identity to $\widetilde\pi_i$ to bound its occupancy-weighted KL divergence from the optimal target policy.

\begin{lemma} \label{dn:lem:soft-regret-KL}\label{lem:kl_regret_equiv}
For every task $i$ and every stationary target policy $\pi$, let $d_2^\pi$ be its normalized occupancy as in~\eqref{dn:eq:target-state-occupancy}, then
\begin{equation}
(1-\gamma_2)\left[
\softvalue_{2,i}(\pi_{2,i}^\star;r_i)
-\softvalue_{2,i}(\pi;r_i)
\right]
=\sum_s d_2^\pi(s)
\operatorname{KL}\!\left(\pi(\cdot\mid s)\,\middle\|\,\pi_{2,i}^\star(\cdot\mid s)\right).
\label{dn:eq:soft-regret-KL}
\end{equation}
\end{lemma}

\begin{proof}
Fix $i$, and let $V_{2,i}^\star(s)$ be the optimal target soft value starting from state $s$ under reward $r_i$.  The soft Bellman equations with entropy coefficient $1$ give
\begin{align*}
V_{2,i}^\star(s)
&=\log\sum_a
\exp\!\left\{r_i(s,a)+\gamma_2\sum_{s'}\Transition_2(s'\mid s,a)V_{2,i}^\star(s')\right\},\\*
\pi_{2,i}^\star(a\mid s)
&=
\exp\!\left\{r_i(s,a)+\gamma_2\sum_{s'}\Transition_2(s'\mid s,a)V_{2,i}^\star(s')-V_{2,i}^\star(s)\right\}.
\end{align*}
The rewards are finite, $\gamma_2<1$, and entropy is bounded by $\log A$, so $V_{2,i}^\star$ is finite and the policy formula gives $\pi_{2,i}^\star(a\mid s)>0$.  Taking logarithms yields
\[
\log\pi_{2,i}^\star(a\mid s)
=r_i(s,a)+\gamma_2\sum_{s'}\Transition_2(s'\mid s,a)V_{2,i}^\star(s')-V_{2,i}^\star(s).
\]
For every action with $\pi(a\mid s)>0$, substituting this expression into the regularized reward under $\pi$ gives
\begin{align*}
r_i(s,a)-\log\pi(a\mid s)
&=V_{2,i}^\star(s)-\gamma_2\sum_{s'}\Transition_2(s'\mid s,a)V_{2,i}^\star(s')
-\log\frac{\pi(a\mid s)}{\pi_{2,i}^\star(a\mid s)}.
\end{align*}
Averaging the last term over $a\sim\pi(\cdot\mid s)$ gives
\[
\sum_a\pi(a\mid s)\log\frac{\pi(a\mid s)}{\pi_{2,i}^\star(a\mid s)}
=\operatorname{KL}\!\left(\pi(\cdot\mid s)\,\middle\|\,\pi_{2,i}^\star(\cdot\mid s)\right),
\]
where actions of zero probability contribute zero.

Now generate a trajectory using $\Transition_2$, $\pi$, and initial distribution $\InitialDist$.  For each finite horizon $T$, conditional expectation gives
\begin{align*}
&\E_{\Transition_2,\pi,\InitialDist}\sum_{t=0}^{T-1}\gamma_2^t \left[V_{2,i}^\star(S_t)-\gamma_2\sum_{s'}\Transition_2(s'\mid S_t,A_t)V_{2,i}^\star(s')\right]\\*
=&\sum_{t=0}^{T-1}\gamma_2^t\E_{\Transition_2,\pi,\InitialDist}V_{2,i}^\star(S_t) -\sum_{t=0}^{T-1}\gamma_2^{t+1}\E_{\Transition_2,\pi,\InitialDist}V_{2,i}^\star(S_{t+1})\\*
=&\sum_s\InitialDist(s)V_{2,i}^\star(s) -\gamma_2^T\E_{\Transition_2,\pi,\InitialDist}V_{2,i}^\star(S_T).
\end{align*}
Since the state and action spaces are finite, the rewards are finite, and $\pi_{2,i}^\star$ is positive, $V_{2,i}^\star$ and the statewise KL terms are bounded.  Thus the second term vanishes as $T\to\infty$, and the discounted sums converge.  Also,
\[
\sum_s\InitialDist(s)V_{2,i}^\star(s)
=\softvalue_{2,i}(\pi_{2,i}^\star;r_i).
\]
Substituting the regularized-reward identity above and rearranging therefore gives
\begin{align*}
&\softvalue_{2,i}(\pi_{2,i}^\star;r_i)
-\softvalue_{2,i}(\pi;r_i)\\
&=\E_{\Transition_2,\pi,\InitialDist}\sum_{t\geq0}\gamma_2^t
\operatorname{KL}\!\left(\pi(\cdot\mid S_t)\,\middle\|\,\pi_{2,i}^\star(\cdot\mid S_t)\right)\\
&=\frac1{1-\gamma_2}\sum_s d_2^\pi(s) \operatorname{KL}\!\left(\pi(\cdot\mid s)\,\middle\|\,\pi_{2,i}^\star(\cdot\mid s)\right).
\end{align*}
The last equality follows by the occupancy identity~\eqref{dn:eq:occupancy-identity}.  Multiplying by $1-\gamma_2$ proves~\eqref{dn:eq:soft-regret-KL}.
\end{proof}

\subsection{The zero-truth case}
\label{dn:app:zero-truth}

This section treats the case when $\Theta^\star =0$ that was excluded in
Assumption~\ref{dn:ass:low-rank}. 
Then $\pi_{1,i}^\star(a\mid s)=\frac1A$
for all $i,s,a$.

\begin{proposition}[Exact recovery when $\Theta^\star=0$]
\label{dn:prop:zero-truth}
Suppose $\Theta^\star=0$. Choose $\lambda_n$
by~\eqref{dn:eq:unweighted-lambda} with $0<\delta<1$.
With probability at least $1-\delta/2$, every minimizer
$\widehat\Theta$ of~\eqref{dn:eq:unweighted-estimator} equals zero.
No restricted-coverage assumption is required.
If Assumption~\ref{dn:ass:normalization} holds, the estimates
in~\eqref{dn:eq:unweighted-U-hat}--\eqref{dn:eq:unweighted-J-hat} also satisfy
\[
\widehat r_i=r_i,
\qquad
\widehat \softvalue_{2,i}(\pi)=\softvalue_{2,i}(\pi;r_i)
\qquad\forall i\in[N],\ \pi\in\Pi
\]
on the same event.
\end{proposition}

\begin{proof}
At $\Theta^\star=0$, the concentration bound~\eqref{dn:eq:matrix-Bernstein-score}
with $x=\log(2(N+SA)/\delta)$ and the tuning~\eqref{dn:eq:unweighted-lambda}
show that the event~\eqref{dn:eq:good-event} holds with probability at least $1-\delta/2$:
\[
\triplenorm{\nabla L_n(0)}_{\op}
\leq\frac{\lambda_n}2.
\]
On this event, fix any minimizer $\widehat\Theta$
of~\eqref{dn:eq:unweighted-estimator}. Since $0\in\mathcal M_B$,
\[
\begin{aligned}
0
&\geq L_n(\widehat\Theta)-L_n(0)
+\lambda_n\triplenorm{\widehat\Theta}_*\\
&\geq\traceinner{\nabla L_n(0)}{\widehat\Theta}
+\lambda_n\triplenorm{\widehat\Theta}_*\\
&\geq\frac{\lambda_n}2
\triplenorm{\widehat\Theta}_*
\geq0.
\end{aligned}
\]
The first inequality is optimality of $\widehat\Theta$;
the second is convexity of $L_n$; the third follows
from~\eqref{dn:eq:score-duality} with $\Theta^\star=0$ and
$\Delta=\widehat\Theta$.
Since $\lambda_n>0$,
these inequalities force $\widehat\Theta=0$.

Equation~\eqref{dn:eq:unweighted-U-hat} now gives
$\widehat u_i(s,a)=-\log A=u_i^\star(s,a)$
for every $i,s,a$.
Under Assumption~\ref{dn:ass:normalization},
Lemma~\ref{dn:lm:irl_reward} and~\eqref{dn:eq:unweighted-r-hat} therefore give
$\widehat r_i=r_i$.
Substituting into~\eqref{dn:eq:unweighted-J-hat} proves the claimed
value equality for every $i\in[N]$ and $\pi\in\Pi$.
\end{proof}

Proposition~\ref{dn:prop:zero-truth} applies to the single-fit estimator in~\eqref{dn:eq:unweighted-estimator}.